\documentclass[twoside]{article}

\usepackage[preprint]{aistats2026}
\usepackage[round]{natbib}

\usepackage{acronym}
\usepackage{amssymb}
\usepackage{amsthm}
\usepackage{thmtools}
\usepackage{thmtools}
\usepackage[capitalise]{cleveref}
\usepackage{subcaption}
\usepackage{csquotes}
\usepackage{dsfont}
\usepackage{tikz}
\usetikzlibrary{shapes.geometric}
\usepackage{enumitem}

\usetikzlibrary{calc,positioning}
\tikzset{
	rv/.style={draw, ellipse},
	pf/.style={draw, rectangle, fill = gray!30},
	arc/.style = {->, >={[round,sep]Stealth}},
}

\newcommand\factorat[4]{
	\node[pf, label={#2:{#3}}](#4) at (#1) {};
}

\newcommand\factor[6]{
	\node[pf, #1=#3 of #2, label={#4:{#5}}](#6) {};
}

\newtheorem{theorem}{Theorem}
\newtheorem{corollary}[theorem]{Corollary}
\newtheorem{definition}{Definition}
\newtheorem{example}{Example}
\newtheorem{lemma}[theorem]{Lemma}

\usepackage{algorithm}
\usepackage{algpseudocode}
\crefname{algorithm}{Alg.}{Algs.}
\Crefname{algorithm}{Algorithm}{Algorithms}
\crefname{definition}{Def.}{Defs.}
\Crefname{definition}{Definition}{Definitions}
\crefname{corollary}{Corollary}{Corollaries}
\Crefname{corollary}{Corollary}{Corollaries}

\acrodef{rv}[randvar]{random variable}
\acrodef{acp}[ACP]{Advanced Colour Passing}
\acrodef{bn}[BN]{Bayesian network}
\acrodef{cp}[CP]{Colour Passing}
\acrodef{fg}[FG]{factor graph}
\acrodef{mn}[MN]{Markov network}
\acrodef{pgm}[PGM]{probabilistic graphical model}
\acrodef{msx}[MSLX]{minimal structural Laplace extension}
\acrodef{lx}[LX]{Laplace extension}
\acrodef{lf}[LF]{Laplace factor}
\acrodef{ms}[MS]{measurable space}

\newcommand{\range}{\ensuremath{\text{range}}}
\newcommand{\scope}{\ensuremath{\text{scope}}}

\allowdisplaybreaks
\algrenewcommand\algorithmicindent{0.5em}
\algtext*{EndIf}
\algtext*{EndFor}

\tikzset{
	rv/.style={draw, ellipse},
	pf/.style={draw, rectangle, draw = black!70, fill = gray!30},
	arc/.style = {->, semithick, >={[round,sep]Stealth}},
	doublearc/.style = {<->, semithick, >={[round,sep]Stealth}},
}

\newcommand\factorcolored[7]{%
	\node[pf, draw=#7, fill=#7!30, #1=#3 of #2, label={#4:{#5}}](#6) {};
}

\begin{document}

%

%

\twocolumn[

\aistatstitle{Inducing Comparability of Factorised Probability Distributions}

\aistatsauthor{ Jan Speller \And Malte Luttermann \And  Marcel Gehrke \And Tanya Braun }

\aistatsaddress{ Data Science Group\\
University of Münster\\ Germany \And Institute for\\ Humanities-Centered AI\\ University of Hamburg\\ Germany \And Institute for\\ Humanities-Centered AI\\ University of Hamburg\\ Germany \And Data Science Group\\ University of Münster\\ Germany } ]

\begin{abstract}
To allow for principled comparison between two probabilistic graphical models defined over non-identical variable sets, they have to be lifted to a common measurable space.
To this end, we propose an extension scheme for any two given models and establish the formal foundation:
Unmatched components are completed using conditionally uniform (Laplace) extensions such that the resulting joint distributions differ from the original ones only by multiplicative constants and coincide under projection. 
This preserves the probabilistic semantics while enabling the application of well-defined distributional discrepancy measures.
We establish the invariance of the induced joint under projection and use the extensions to provide a minimal structural extension of two factor graphs to the smalles common measurable space as well as to a common graphical structure by a deterministic algorithm.
In addition, we discuss structural and measure-theoretic properties and identify promising criteria for comparison methodologies.
\end{abstract}

\section{Introduction}
Comparing probability distributions assumes that the distributions are defined over the same support (e.g., total variance distance~\citep{Bretagnolle1978estimation}), with compatible metric spaces (e.g., Wasserstein metric~\citep{vaserstein1969markov}), or ideally the same \ac{ms} (e.g., Hellinger distance~\citep{hellinger1909neue}).
However, in many settings, one faces the challenge of having to compare distributions that are defined over non-identical \acp{ms}.
Changes might come from concept drifts over time \citep{HadBeWe14,FiBrGeMo21a}, updates in human-aware settings \citep{ChaSrZhKa17,KulSrKa19}, or simply from learning two models from two related sources (e.g., the databases of two companies that contain similar but not identical information) or with different learning algorithms.
Since such models can only be compared after lifting them to a common \ac{ms}, we study principled distribution extensions that preserve semantics.

Specifically, we focus on \acp{fg} \citep{Frey1997a} representing factorised probability distributions as a generalised problem setting to enable both a global comparison of full joint distributions \citep{kullback1951information} as well as a local comparison on factor level \citep{Chan2005a}. 
A full joint distribution can always be considered an \ac{fg} with a single factor.
Factorised distributions exploit (conditional) independences among \acp{rv} to store the same full joint distribution with fewer entries, shifting complexity-wise from $O(r^n)$ with $r$ being the largest number of values a \ac{rv} can take and $n$ being the number of \acp{rv} in the full joint distribution to $O(m r^s)$ with $m$ being the number of factors and $s$ being the largest number of \acp{rv} in a factor.
Such factorised distributions are often accompanied by a graphical representation, making them fall into the category of \acp{pgm}.
There are several flavours of \acp{pgm} such as the above mentioned \acp{fg}, \acp{bn} \citep{Pea88} as well as \acp{mn} \citep{Moussouris1974a}.
According to the Hammersley-Clifford theorem~\citep{HamCl71}, every underlying probability distribution encoded by an \ac{fg} can be represented by a \ac{bn} and an \ac{mn}.
The results presented in this paper are thus not limited to \acp{fg} but can also be used to compare factorised distributions encoded by \acp{bn} and \acp{mn}.

To extend \acp{fg} in a principled way, we consider their structure, adding new factors, new \acp{rv}, or existing \acp{rv} to existing factors.
To actually preserve the semantics of an \ac{fg} under projection while extending it, we employ uniform (Laplace) extensions to complete unmatched components between factors of different models.
A Laplace extension represents a special case of a general extension in the field of probability theory as a way to extend probability spaces~\citep{bierlein1962fortsetzung,ascherl1977two,bogachev2007measure}, which has been comparatively underexplored.
We prove that Laplace extensions admit a surjective, measure-preserving projection onto the original \ac{fg}. Based on this result, we introduce \acp{msx} to align \acp{fg} with non-identical \ac{rv} sets, enabling comparison via existing distance measures. To the best of our knowledge, this is the first principled approach for comparing \acp{pgm} with non-identical \ac{rv} sets.

The remainder of this paper is structured as follows:
After notations, we present \ac{fg} extensions, followed by an algorithm that guarantees extensions of two arbitrary \acp{fg} defined on the same \ac{ms} while enforcing an identical graphical structure, thereby enabling direct comparison, followed by a discussion and conclusion.
Longer proofs and a discussion of minimality are provided in the appendix.

\section{Notation}
Given a set $\boldsymbol{R}$ of \acp{rv}, let $\mathcal{X}_{\boldsymbol{R}}:=\times_{X\in\boldsymbol{R}}\ \range(X)$ denote the Cartesian product of their ranges, where $\range(X)$ is the set of values that $X$ can take.
An \ac{fg} is a \ac{pgm} that compactly encodes a probability distribution over a set of \acp{rv} by factorising the distribution into a product of factors~\citep{Frey1997a,KscFrLo01}.

\begin{definition}[Factor Graph] \label{def:fg}
	A \textbf{\acf{fg}} $M = (\boldsymbol{V}, \boldsymbol{E})$ is an undirected bipartite graph consisting of a set of nodes $\boldsymbol{V} = \boldsymbol{R} \cup \boldsymbol{\Phi}$, where $\boldsymbol{R} = \{X_1, \ldots, X_n\}$ is a set of \acp{rv} and $\boldsymbol{\Phi} = \{\phi_1, \ldots, \phi_m\}$ is a set of factors (functions), as well as a set of edges $\boldsymbol{E} \subseteq \boldsymbol{R} \times \boldsymbol{\Phi}$.  
    There exists an edge between a \ac{rv} $X_i \in \boldsymbol{R}$ and a factor $\phi_j \in \boldsymbol{\Phi}$ in $\boldsymbol{E}$ if $X_i$ appears in the argument list (also called \emph{scope}) $\boldsymbol{R}_{(j)}:=\scope(\phi_j)$ of $\phi_j$, where $\boldsymbol{R}_{(j)} \subseteq \boldsymbol{R}$.
	A factor $\phi_j$ defines a function
        $\phi_j \colon \mathcal{X}_{\boldsymbol{R}_{(j)}} \mapsto \mathbb{R}_{>0}$
    that maps range values of its arguments to a positive real number, called potential.
	We define the joint potential for an assignment $\boldsymbol{r}$ (with $\boldsymbol{r}$ abbreviating $\boldsymbol{R} = \boldsymbol{r}$) as
		$\psi(\boldsymbol{r}) = \prod_{ \phi_j\in\boldsymbol{\Phi} } \phi_j(\boldsymbol{r}_{j})$,
	where $\boldsymbol {r}_j$ is a projection of the assignment $\boldsymbol{r}$ to the scope $\boldsymbol{R}_{(j)}$ of $\phi_j$.
    Given the \ac{ms} $(\mathcal{X}_{\boldsymbol{R}},\mathcal{P}(\mathcal{X}_{\boldsymbol{R}}))$, where $\mathcal{X}_{\boldsymbol{R}}$ is the set of all possible assignments and $\mathcal{P}(\mathcal{X}_{\boldsymbol{R}})$ being its power set serving as the $\sigma$-algebra, the probability measure $P_M$ is the normalised joint potential
	\begin{align*}
		P_M(\boldsymbol{r}) = \frac{1}{Z} \prod_{ \phi_j\in\boldsymbol{\Phi} }\phi_j(\boldsymbol{r}_{ j}) = \frac{1}{Z} \,\psi(\boldsymbol{r}),
	\end{align*}
    where $Z = \sum_{\boldsymbol{r}\in\mathcal{X}_{\boldsymbol{R}}} \prod_{ \phi_j\in\boldsymbol{\Phi} }\phi_j(\boldsymbol{r}_{ j})$ is the normalisation constant (also called partition function).
\end{definition}

\begin{example}[Factor Graph]
    Consider the \ac{fg} $M_{ex} = (\boldsymbol{R} \cup \boldsymbol{\Phi}, \boldsymbol{E})$ depicted in \cref{fig:examplefg}, where $\boldsymbol{R} = \{ A, \allowbreak B, \allowbreak C \}$, $\boldsymbol{\Phi} = \{ \phi_1, \allowbreak \phi_2 \}$, and $\boldsymbol{E} = \{ (A, \allowbreak \phi_1), \allowbreak (B, \allowbreak \phi_1), \allowbreak (B, \allowbreak \phi_2), \allowbreak (C, \allowbreak \phi_2) \}$.
    The scopes are given by $\boldsymbol{R}_{(1)} = \{ A, \allowbreak B \}$ and $\boldsymbol{R}_{(2)} = \{ C, \allowbreak B \}$.
    The function definitions of $\phi_1$ and $\phi_2$ are given in the tables in \cref{fig:examplefg}, that is, $\phi_1(A = \text{true}, B = \text{true}) = \varphi_1$, $\phi_1(A = \text{true}, B = \text{false}) = \varphi_2$, and so on, where $\varphi_i \in \mathbb{R}_{>0}$, $i \in \{ 1, \allowbreak \ldots, \allowbreak 8 \}$, are positive real numbers.
    The joint potential, e.g., for the assignment $\boldsymbol r = (A = \text{true}, \allowbreak B = \text{true}, \allowbreak C = \text{true})$ is given by $\psi(\boldsymbol{r}) = \phi_1(A = \text{true}, B = \text{true}) \cdot \phi_2(C = \text{true}, B = \text{true}) = \varphi_1 \cdot \varphi_5$.
\end{example}

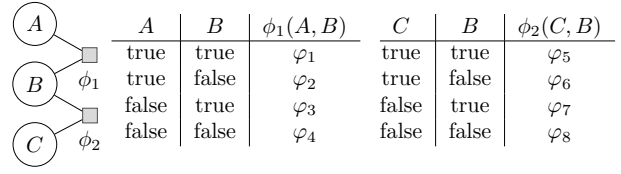
\begin{figure}[t]
    \centering
    \resizebox{\linewidth}{!}{
        \begin{tikzpicture}[scale=0.9]
            \node[circle, draw] (A) {$A$};
        	\node[circle, draw] (B) [below = 0.25cm of A] {$B$};
        	\node[circle, draw] (C) [below = 0.25cm of B] {$C$};
        	\factor{below right}{A}{0.125cm and 0.5cm}{270}{$\phi_1$}{f1}
        	\factor{below right}{B}{0.125cm and 0.5cm}{270}{$\phi_2$}{f2}
            \node[right = 0.1cm of f1, yshift=-4mm] (tabs) {
                \begin{tabular}{c|c|c c@{\hspace{0.1cm}}c|c|c}
                    $A$ & $B$ & $\phi_1(A,B)$ & & $C$ & $B$ & $\phi_2(C,B)$ \\ \cline{1-3} \cline{5-7}
                    true  & true  & $\varphi_1$ & & true  & true  & $\varphi_5$ \\
                    true  & false & $\varphi_2$ & & true  & false & $\varphi_6$ \\
                    false & true  & $\varphi_3$ & & false & true  & $\varphi_7$ \\
                    false & false & $\varphi_4$ & & false & false & $\varphi_8$ \\
                \end{tabular}
            };

        	\draw (A) -- (f1);
        	\draw (B) -- (f1);
        	\draw (B) -- (f2);
        	\draw (C) -- (f2);
        \end{tikzpicture}
    }
    \caption{An exemplary \ac{fg} encoding a probability distribution over three \acp{rv} $A$, $B$, and $C$ (left). The function definitions of the factors $\phi_1$ and $\phi_2$, denoted as potential tables, are given on the right.}
    \label{fig:examplefg}
\end{figure}

We next introduce extensions of \acp{fg} for comparing models defined over non-identical variable sets.

\section{Factor Graph Extensions}\label{section:three}
In general, an extension to an \ac{fg} can add new factors, new \acp{rv}, or edges (i.e., adding an existing \acp{rv} to an existing factor), resulting in an enlarged \ac{ms} over which the \ac{fg} is defined, whenever a new \ac{rv} is added.
The purpose may be to partially update an \ac{fg} while preserving an accurate description of the underlying distribution.
For comparability, the goal is not to add information but to structurally align an \ac{fg} with another, ensuring a common \ac{ms} while preserving the original distribution under projection onto the original set of \acp{rv}.
After presenting general \ac{fg} extensions, we apply the concept of uniform variable influence as a special case to ensure desirable comparability properties.

\subsection{General Factor Graph Extensions}\label{subsec:generalfgx}
This section defines general extensions for \acp{fg}, which focuses on structural relations (\acp{rv}, factors) without imposing any constraints on the potentials in the factors.

\begin{definition}[Factor Graph Extension]\label{def:factorgraphextension}
    An \textbf{extension} $M^{\text{x}}=(\boldsymbol V^{\text{x}},\boldsymbol E^{\text{x}}) = (\boldsymbol{R}^{\text{x}}\cup \boldsymbol\Phi^{\text{x}}, \boldsymbol E^{\text{x}})$ of an \ac{fg} $M = (\boldsymbol{V}^{\text{orig}}, \boldsymbol{E}^{\text{orig}})=(\boldsymbol{R}^{\text{orig}}\cup\boldsymbol{\Phi}^{\text{orig}},\boldsymbol{E}^{\text{orig}})$ is any \ac{fg} $M^{\text{x}}$, where\vspace{-0.25cm}
    \begin{itemize}
        \item [(i)] $\boldsymbol{R}^{\text{x}}= \boldsymbol{R}^{\text{orig}} \cup \boldsymbol{R}^{\text{new}}$ with $\boldsymbol{R}^{\text{orig}} \cap \boldsymbol{R}^{\text{new}} =\emptyset$,
        \item [(ii)] $\boldsymbol{\Phi}^{\text{x}}=\boldsymbol{\Phi}^{\text{orig},\text{x}} \cup \boldsymbol{\Phi}^{\text{new}}$ with $\boldsymbol{\Phi}^{\text{orig},\text{x}} \cap \boldsymbol{\Phi}^{\text{new}}=\emptyset$ and $\boldsymbol{\Phi}^{\text{orig},\text{x}}$ being any set of factors for which there exists a bijection
$\eta:\boldsymbol{\Phi}^{\text{orig}}\to\boldsymbol{\Phi}^{\text{orig},\text{x}}$
such that $\boldsymbol{R}_{(i)}^{\text{orig}}\subseteq\boldsymbol{R}_{(j)}^{\text{x}}$ whenever $\eta(\phi_i)=\phi_j^{\text{x}}$.
    \end{itemize}\vspace{-0.25cm}
    The set of edges $\boldsymbol E^{\text{x}}$ contains an edge between a \ac{rv} $X \in \boldsymbol{R}^{\text{x}}$ and a factor $\phi_j^{\text{x}} \in \boldsymbol{\Phi}^{\text{x}}$ if $X \in \boldsymbol{R}_{(j)}^{\text{x}}$.
    For the \emph{trivial} extension, it holds that $\boldsymbol{R}^{\text{new}}= \emptyset$, $\boldsymbol{\Phi}^{\text{new}}=\emptyset$, and $\boldsymbol{\Phi}^{\text{orig},\text{x}} = \boldsymbol{\Phi}^{\text{orig}}$, yielding $M^{\text{x}}=M$. 
\end{definition}



\begin{example}[Factor Graph Extension]
    Consider the \ac{fg} $M_{ex}$ depicted in \cref{fig:examplefg} and assume that only an additional edge $\{ C, \allowbreak \phi_1 \}$ is added to $M_{ex}$ to obtain an extension $M^{\text{x}}_{ex}$ of $M_{ex}$. 
    Then, the scope of $\phi_1$ extends from $\boldsymbol{R}_{(1)}^{\text{orig}} = \{ A, \allowbreak B \}$ to $\boldsymbol{R}_{(1)}^{\text{x}} = \{ A, \allowbreak B, \allowbreak C \}$.
    We get $\boldsymbol{\Phi}^{\text{orig},\text{x}} = \{ \phi_1^{\text{x}} \} \cup \left(\boldsymbol{\Phi}^{\text{orig}} \setminus \{ \phi_1 \} \right)$, where $\phi_1^{\text{x}}(A, \allowbreak B, \allowbreak C)$ now defines a potential table with $2^3 = 8$ instead of $2^2 = 4$ entries, with $\boldsymbol{R}^{\text{new}} = \emptyset$ and $\boldsymbol{\Phi}^{\text{new}} = \emptyset$.
\end{example}

As soon as any edge is added to any of the original factors $\phi_i\in\boldsymbol{\Phi}^{\text{orig}}$ (scope extension), its number of potentials grows with regard to the number of range values of the additional \ac{rv}.
In general, this means that $\boldsymbol{\Phi}^{\text{orig},\text{x}}$ is not a subset nor a superset of the original set of factors $\boldsymbol{\Phi}^{\text{orig}}$.
Additionally, although $M$ is not necessarily a subgraph of $M^{\text{x}}$ in the strict graph-theoretic sense, the extension preserves the original factorisation of the full joint probability distribution in the sense that for every factor in $M$, there is a factor in $M^{\text{x}}$ whose scope contains the scope of the original factor.

There are different ways in which a non-trivial extension of a \ac{fg} can be realised.
To clearly distinguish between the different cases and to make dependencies explicit, we represent the scope of an extended factor $\phi_i^{\text{x}}$ as a partition of \acp{rv}
\begin{align*}
    \boldsymbol{R}^{\text{x}}_{(i)} := \scope(\phi_i^{\text{x}}) 
    := \boldsymbol{R}_{(i)}^{\text{orig}} \,\dot\cup\, \boldsymbol{R}_{(i)}^{\text{cross}} \,\dot\cup\, \boldsymbol{R}_{(i)}^{\text{new}},
\end{align*}
where $\boldsymbol{R}_{(i)}^{\text{orig}}=\scope(\phi_i)$ is the scope of the original factor $\phi_i$, $\boldsymbol{R}_{(i)}^{\text{new}}\subseteq\boldsymbol{R}^{\text{new}}$ are new \acp{rv} that have not been part of the set $\boldsymbol{R}^{\text{orig}}$ of the original \ac{fg} $M$ before, and $\boldsymbol{R}_{(i)}^{\text{cross}}\subseteq\boldsymbol{R}^{\text{orig}}\setminus\boldsymbol{R}_{(i)}^{\text{orig}}$ are \acp{rv} that have been in the scope of at least one factor $\phi_j \in \boldsymbol{\Phi}^{\text{orig}}$, $j\neq i$, but not in the scope of $\phi_i$.

The projection $\boldsymbol{r}_i^{\text{x}}$ of any assignment $\boldsymbol{r}^{\text{x}} \in \mathcal{X}_{\boldsymbol{R}^{\text{x}}}$ in the extension to the scope $\boldsymbol{R}^{\text{x}}_{(i)}$ of $\phi_i$ can also be partitioned as $\boldsymbol{r}_i^{\text{x}}= (\boldsymbol{r}^\text{orig}_i,\boldsymbol{r}^\text{cross}_i,\boldsymbol{r}^\text{new}_i)$ in
\begin{align*}
    \mathcal{X}_{\boldsymbol{R}^{\text{orig}}_{(i)}} \times \mathcal{X}_{\boldsymbol{R}^{\text{cross}}_{(i)} }\times \mathcal{X}_{\boldsymbol{R}^{\text{new}}_{(i)} }=\mathcal{X}_{\boldsymbol{R}_{(i)}^{\text{x}}}.
\end{align*}
The partition is consistent with the previous notation of the \acp{rv} $\boldsymbol{R}^{\text{new}} = \cup_{i:\phi^{\text{x}}_i\in\boldsymbol{\Phi}^{\text{new}}} \boldsymbol{R}_{(i)}^{\text{new}} $ and $\boldsymbol{R}^{\text{orig}} = \cup_{i:\phi_i\in\boldsymbol{\Phi}^{\text{orig}}}\boldsymbol{R}_{(i)}^{\text{orig}} $, where the individual scopes are not required to be disjoint.

Generally, an extension is obtained by applying one or more of the following elementary extensions:\vspace{-0.25cm}
\begin{enumerate}[label=(\roman*)]
    \item \textbf{Adding a new factor:} A new factor $\phi_i^{\text{x}} \in \boldsymbol{\Phi}^{\text{new}}$ is introduced, whose scope may contain new \acp{rv} $\boldsymbol{R}_{(i)}^{\text{new}}$ or original \acp{rv} $\boldsymbol{R}_{(i)}^{\text{cross}}$, whereas $\boldsymbol{R}_{(i)}^{\text{orig}}=\emptyset$.\vspace{-0.1cm}
    \item \textbf{Adding a new \ac{rv} to an original factor:} A \ac{rv} $X$ that did not occur in the scope of any original factor is added to the scope of $\phi_i$ ($X \in \boldsymbol{R}_{(i)}^{\text{new}}$).\vspace{-0.1cm}
    \item \textbf{Adding an original \ac{rv} to an original factor:} A \ac{rv} $X$ that occurred in the scope of at least one original factor but not in the scope of $\phi_i$ is added to the scope of $\phi_i$ ($X \in \boldsymbol{R}_{(i)}^{\text{cross}}$).\vspace{-0.1cm}
\end{enumerate}
The edges in the extended \ac{fg} are induced by the scopes of the factors as in the \ac{fg} definition in \cref{def:fg}.
Any extension of an \ac{fg} necessarily contains at least as many potentials per factor and may introduce additional factors.
This increased representational capacity is not tied to any specific distributional form.
However, next, we consider the form of uniform extensions for the purpose of comparability later on.

\subsection{Laplace Extension}
Extensions, in general, allow for any potentials in the extended parts.
If constructed appropriately, though, the extension admits the original joint distribution as a surjective projection, using the discrete uniform distribution.
To this end, we introduce the concept of a \ac{lx}.
\begin{definition}[Laplace Factor]\label{def:lapfactor}
    A factor $\phi_{i}$ is a \textbf{\ac{lf}} if there exists a constant $c_i\in \mathbb{R}_{>0}$ such that
    \begin{align*}
        \forall\;\boldsymbol{r}^{\text{orig}}_{ i} \in \mathcal{X}_{\boldsymbol{R}_{(i)}^{\text{orig}}} : \phi_{i}(\boldsymbol{r}^{\text{orig}}_{ i})=c_i.
    \end{align*}
\end{definition}
 As a special case of the extension options introduced in \cref{subsec:generalfgx}, consider elementary extension~(i).
 Using the established notation, an extended \ac{lf} satisfies
 \begin{align*}
         \phi_{i}^{\text{x}}(\boldsymbol{r}^{\text{x}}_{i})=c_i \text{ for all }\boldsymbol{r}^{\text{x}}_{i} \in \mathcal{X}_{\boldsymbol{R}_{(i)}^{\text{x}}}.
     \end{align*}
The remaining elementary extensions~(ii) and~(iii) of \cref{subsec:generalfgx} are captured by the following definition.
\begin{definition}[Laplace Extension]
    Let $M$ be an \ac{fg} and let $M^{\text{x}}$ be an extension of $M$.
    A factor $\phi_i^{\text{x}}\in\boldsymbol{\Phi}^{\text{x}}$ is a \textbf{\acf{lx}} of a given factor $\phi_i\in\boldsymbol{\Phi}^{\text{orig}} $ if there exists a constant $c_i\in \mathbb{R}_{>0}$ such that
    \begin{align*}
        c_i\cdot \phi_i(\boldsymbol{r}^{\text{orig}}_{i})=  \phi_i^{\text{x}}(\boldsymbol{r}_i^{\text{x}})\text{ for all }\boldsymbol{r}_i^{\text{x}} \in \mathcal{X}_{\boldsymbol{R}_{(i)}^{\text{x}}}.
    \end{align*}       
    An extension $M^{\text{x}}$ of an \ac{fg} $M$ is called \ac{lx} if every extension $\phi_{i}^{\text{x}}$ of a factor $\phi_{i}\in\boldsymbol{\Phi}^{\text{orig}} $ is an \ac{lx}  and every factor $\phi_i^{\text{x}}\in\boldsymbol{\Phi}^{\text{new}}$ is a \ac{lf}.
    An \ac{fg} is a \textbf{Laplace \ac{fg}} if all of its factors are \acp{lf}.
\end{definition}

Under an \ac{lx}, the introduced set of \acp{rv} $\boldsymbol{R}_{(i)}^{\text{cross}}\cup \boldsymbol{R}_{(i)}^{\text{new}}$ is locally conditionally independent of the remaining \acp{rv} within the same factor $\phi_i^{\text{x}}$.
At first glance, this construction appears to introduce unnecessary dependencies by augmenting the \ac{fg} with seemingly non-informative parts.
However, \acp{lx} act as a theoretical device for enabling comparability between \acp{fg} on a factor level.
Importantly, the influence of a \ac{rv} is factor-specific: a variable may be Laplace (i.e., non-informative) in one factor while being highly informative in another and thus might not be uniformly distributed with respect to the full joint distribution.

\begin{example}[Laplace Extension]
    Consider the \ac{fg} $M$ (left) and its extension $M^{\text{x}}$ (right) shown in \cref{fig:examplefgext}.
    Let $\phi_1^{\text{x}}(A = a, \allowbreak B = b, \allowbreak D = d) = \phi_1^{\text{orig}}(A = a, \allowbreak B = b)$ for all assignments $(a,b)$ of $A$ and $B$ independent of the assigned value $d$ of $D$ (i.e., $\phi_1^{\text{x}}(A = \text{true}, \allowbreak B = \text{true}, \allowbreak D = \text{true}) = \phi_1^{\text{orig}}(A = \text{true}, \allowbreak B = \text{true}) = \varphi_1$, $\phi_1^{\text{x}}(A = \text{true}, \allowbreak B = \text{true}, \allowbreak D = \text{false}) = \phi_1^{\text{orig}}(A = \text{true}, \allowbreak B = \text{true}) = \varphi_1$, and so on).
    Further, let $\phi_2^{\text{x}}$ be a \ac{lf} (e.g., $\phi_2^{\text{x}}(B = b, \allowbreak C = c) = 1$ for all assignments $(b,c)$ of $B$ and $C$).
    Then, $M^{\text{x}}$ is an \ac{lx} of $M$.
\end{example}

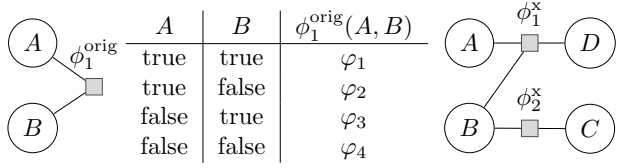
\begin{figure}[t]
    \centering
    \resizebox{\linewidth}{!}{
        \begin{tikzpicture}[scale=0.9]
        	\node[circle, draw] (A) {$A$};
        	\node[circle, draw] (B) [below = 0.5cm of A] {$B$};
        	\factor{below right}{A}{0.25cm and 0.5cm}{90}{$\phi_1^{\text{orig}}$}{f1}
        	\draw (A) -- (f1);
        	\draw (B) -- (f1);
    
            \node[right = 0.2cm of f1] (tabs) {
                \begin{tabular}{c|c|c}
                    $A$ & $B$ & $\phi_1^{\text{orig}}(A,B)$ \\ \hline
                    true  & true  & $\varphi_1$ \\
                    true  & false & $\varphi_2$ \\
                    false & true  & $\varphi_3$ \\
                    false & false & $\varphi_4$ \\
                \end{tabular}
            };
            
            \node[circle, draw] (A2) [right = 5.5cm of A]{$A$};
            \node[circle, draw] (D2) [right = 1cm of A2]{$D$};
        	\node[circle, draw] (B2) [below = 0.5cm of A2] {$B$};
        	\node[circle, draw] (C2) [below = 0.5cm of D2] {$C$};
            \factorat{$(A2)!0.5!(D2)$}{90}{$\phi_1^{\text{x}}$}{f12}
            \factorat{$(B2)!0.5!(C2)$}{90}{$\phi_2^{\text{x}}$}{f22}
        	\draw (A2) -- (f12);
        	\draw (B2) -- (f12);
        	\draw (B2) -- (f22);
        	\draw (C2) -- (f22);
            \draw (D2) -- (f12);
        \end{tikzpicture}
    }
    \caption{An \ac{fg} $M$ encoding a full joint probability distribution over two \acp{rv} $A$ and $B$ (left) and an extension $M^{\text{x}}$ of $M$ (right).}
    \label{fig:examplefgext}
\end{figure}

General elementary extensions induce a modification of the partition function $Z^{\text{x}}$ in the extended \ac{fg} relative to the original partition function~$Z$ that are hard to capture and may require a costly recomputation of the partition function.
For the special case of an \ac{lx}, the effect on the partition function can be explicitly characterised.

\begin{theorem}\label{theorem:generalpartitionfunction}
    Let $M^{\text{x}}$ be an \ac{lx} of $M$.
    Then, the partition function of $P_{M^{\text{x}}}$ is given by
    \begin{align*}
            Z^{\text{x}}= c \cdot \prod_{X\in \boldsymbol{R}^{\text{new}}}^{} \vert \range(X)\vert \cdot Z= c\cdot \vert \mathcal{X}_{\boldsymbol{R}^{\text{new}}}\vert\cdot Z
    \end{align*}
    with $c:=\prod_{\phi^{\text{x}}_i\in\boldsymbol{\Phi}^{\text{x}}} c_i\in\mathbb{R}_{> 0}$ and $c_i$ given by $c_i\cdot \phi_i(\boldsymbol{r}^{\text{orig}}_i)=  \phi_i^{\text{x}}(\boldsymbol{r}_i^{\text{x}})$ for $\phi_i^{\text{x}}\in\boldsymbol{\Phi}^{\text{orig},\text{x}}$ and $\phi_{i}(\boldsymbol{r}^{\text{x}}_{i})=c_i$ for $\phi_i^{\text{x}}\in\boldsymbol{\Phi}^{\text{new}}$.
\end{theorem}
\begin{proof}[Proof Sketch]
    Due to normalisation of probability measures, we get 
	\begin{align*}
		\hspace{-0.7cm}1 &=\sum_{\boldsymbol{r}^{\text{x}}\in\mathcal{X}_{\boldsymbol{R}^{\text{x}}}}P_{M^{\text{x}}}(\boldsymbol{r}^{\text{x}})\\
        &\overset{\text{Lap.}}{=}\frac{1}{Z^{\text{x}}}\sum_{\boldsymbol{r}^{\text{x}}\in\mathcal{X}_{\boldsymbol{R}^\text{x}}}\prod_{\phi_i^{\text{x}}\in \boldsymbol{\Phi}^{\text{orig},\text{x}}}^{} c_i\cdot \phi_i(\boldsymbol{r}^{\text{orig}}_i)\prod_{\phi_j^{\text{x}}\in \boldsymbol{\Phi}^{\text{new}}}^{} c_j\\
        &\overset{\text{ind.}}{=}\frac{1}{Z^{\text{x}}}\prod_{\phi_j^{\text{x}}\in \boldsymbol{\Phi}^{\text{x}}}^{} c_j\hspace*{-1em}
        \sum_{\boldsymbol{r}^{\text{orig}}\in\mathcal{X}_{\boldsymbol{R}^{\text{orig}}}}\sum_{\boldsymbol{r}^{\text{new}}\in\mathcal{X}_{\boldsymbol{R}^{\text{new}}} }\prod_{\phi_i^{\text{x}}\in \boldsymbol{\Phi}^{\text{orig},\text{x}}}^{}\hspace*{-0.5em}\phi_i(\boldsymbol{r}^{\text{orig}}_i)\\
           &=\frac{1}{Z^{\text{x}}}\cdot c \cdot \vert \mathcal{X}_{\boldsymbol{R}^{\text{new}}}\vert 
        \sum_{\boldsymbol{r}^{\text{orig}}\in\mathcal{X}_{\boldsymbol{R}^{\text{orig}}}}\prod_{\phi_i\in \boldsymbol{\Phi}^{\text{orig}}}^{}\phi_i(\boldsymbol{r}^{\text{orig}}_i)\\
        &\overset{\text{def.}}{=}\frac{1}{Z^{\text{x}}}\cdot c
         \cdot \vert \mathcal{X}_{\boldsymbol{R}^{\text{new}}}\vert \cdot Z\\
 &\Leftrightarrow Z^{\text{x}}= c  \cdot \vert \mathcal{X}_{\boldsymbol{R}^{\text{new}}}\vert \cdot Z
        =c  \cdot\hspace{-0.25cm} \prod_{X\in \boldsymbol{R}^{\text{new}}}^{}\vert \range(X)\vert \cdot Z\tag*{\qedsymbol}
        \end{align*}
       \renewcommand{\qedsymbol}{}
\end{proof}\vspace{-0.5cm}

This characterisation shows that the change in the partition function induced by an \ac{lx} decomposes into a purely combinatorial term, determined by the cardinalities of the newly introduced variables, and rescaling scalars $c_i \in\mathbb{R}_{> 0}$ induced by \acp{lf} and \acp{lx}.
Consequently, the effect of the extension on normalisation does not depend on a specific assignment of the original \acp{rv}.
For the general case, where we do not have a constant influence on a factor, this is not necessarily true. 
The upcoming corollary follows directly from \cref{theorem:generalpartitionfunction}.
\begin{corollary}
    If $M^{\text{x}}$ is an \ac{lx} of $M$ and $\boldsymbol{R}^{\text{new}}=\emptyset$, then $Z^{\text{x}} = c \cdot Z$, where $c \in \mathbb{R}_{> 0}$ is a scalar.
\end{corollary}

This implies that, for a constructed \ac{lx}, the rescaling constants can be chosen explicitly such that the resulting normalisation depends solely on the involved cardinalities.
This is particularly convenient for a Laplace \ac{fg}, which may therefore serve as a natural reference.

\begin{corollary}\label{corollary:lfg}
    Let $M$ be a Laplace \ac{fg} and let $\kappa_X:=\lvert\{\phi_i: X\in \boldsymbol{R}^{\text{orig}}_{(i)}\}\rvert$ denote the number of factors in which \ac{rv} $X$ appears.
    Then, $M$ can equivalently be represented by scaling all factors $\phi_i$ in $M$ such that
    \begin{align}
        \phi_i(\boldsymbol{r}_i^{\text{orig}})=\prod_{X\in \boldsymbol{R}^{\text{orig}}_{(i)}}\vert \range(X) \vert ^{- 1/\kappa_X } \label{eq:lfg}
    \end{align}
    for all assignments $\boldsymbol{r}^{\text{orig}}_i\in\mathcal{X}_{\boldsymbol{R}^{\text{orig}}_{(i)}}$ and with $Z=1$. 
\end{corollary}
%
\begin{proof}[Proof Sketch]
    Since $M$ is a Laplace \ac{fg}, all factors are Laplace, which means every factor $\phi_i$ has the same value for all its potentials and there exists a constant $c_i\in \mathbb{R}_{>0}$ such that 
        $\phi_{i}(\boldsymbol{r}^{\text{orig}}_{i})=c_i\text{ for all }\boldsymbol{r}^{\text{orig}}_i\in \mathcal{X}_{\boldsymbol{R}^{\text{orig}}_{(i)}}$.
    We choose $c_i:=\prod_{X\in \boldsymbol{R}^{\text{orig}}_{(i)}}\vert X\vert ^{- 1/\kappa_X } $ for the $i$-th factorand and use the independence of the sets to show that the choice normalises $Z$ to $1$.
\end{proof}
\cref{eq:lfg} can be interpreted as the relative mass or weight of an individual factor with respect to the entire \ac{fg}.
In addition, for any given \ac{fg} $M$, there always exists a Laplace \ac{fg} of the form given in \cref{corollary:lfg} that has the same structure as $M$, while being Laplace, and thus can serve as a natural reference for comparison.
Although this construction may appear unnecessarily elaborate compared to ignoring normalisation altogether, it enables a principled comparison of the influence of individual factors relative to the remainder of the \ac{fg}.
Specifically, it allows deviations of each factor from its Laplace counterpart to be quantified on a common scale.

\paragraph{Surjection.}
We now formalise the fact that an \ac{lx} induces a surjective, measure-preserving projection from the extended probability space onto the original one.
By extending the state space $\mathcal{X}_{\boldsymbol{R}^{\text{orig}}}$ to $\mathcal{X}_{\boldsymbol{R}^{\text{x}}}=\mathcal{X}_{\boldsymbol{R}^{\text{orig}}}\times \mathcal{X}_{\boldsymbol{R}^{\text{new}}}$, the canonical projection 
\begin{align*}
    \pi \colon \mathcal{X}_{\boldsymbol{R}^{\text{x}}}&\mapsto\mathcal{X}_{\boldsymbol{R}^{\text{orig}}}\text{ with }(\boldsymbol{r}^{\text{orig}},\boldsymbol{r}^{\text{new}})\mapsto \boldsymbol{r}^{\text{orig}}
\end{align*}
is surjective with preimage
\begin{align*}
    \pi^{-1}(\boldsymbol{r}^{\text{orig}})= \bigcup_{\boldsymbol{r}^{\text{new}}\in\mathcal{X}_{\boldsymbol{R}^\text{new}}}\{(\boldsymbol{r}^{\text{orig}},\boldsymbol{r}^{\text{new}})\}.
\end{align*}
The original probability distribution $P_M$ can thus be written as a composition, as shown in the upcoming theorem.
\begin{theorem}\label{theorem:surjection}
    Let $M^{\text{x}}$ be an \ac{lx} of an \ac{fg} $M$. Then,
    \begin{align}
              P_M= P_{M^{\text{x}}}\;\circ\;\pi^{-1}. \label{equ:surjection}
    \end{align}
\end{theorem}
One interpretation of \cref{equ:surjection} is that the probability measure $P_M$ is preserved under \ac{lx} and every assignment within the original (smaller) probability space is represented as the preimage of the surjective projection $\pi$ in $\mathcal{P}(\mathcal{X}_{\boldsymbol{R}^{\text{x}}})$ as follows:
\begin{align*}
   P_{M^{\text{x}}}\left(\pi^{-1}\left(\boldsymbol{r}^{\text{orig}}\right) \right)= P_M(\boldsymbol{r}^{\text{orig}})\text{ for all }\boldsymbol{r}^{\text{orig}}\in\mathcal{X}_{\boldsymbol{R}^{\text{orig}}}.
\end{align*}
\begin{proof}
It is sufficient to apply equality transformations using the given properties and \cref{theorem:generalpartitionfunction}.
\begin{align*}
P&_{M^{\text{x}}}(\pi^{-1}\left(\boldsymbol{r}^{\text{orig}}\right))
        \overset{\text{dis.}}{=} \sum_{\boldsymbol{r}^{\text{new}}\in\mathcal{X}_{\boldsymbol{R}^{\text{new}}}} P_{M^{\text{x}}}\left(\{(\boldsymbol{r}^{\text{orig}},\boldsymbol{r}^{\text{new}})\}\right)\\
	    &\overset{\text{def.}}{=}\frac{1}{Z^{\text{x}}} \sum_{\boldsymbol{r}^{\text{new}}\in\mathcal{X}_{\boldsymbol{R}^{\text{new}}}}\prod_{\phi_{i}^{\text{x}}\in \boldsymbol{\Phi}^{\text{x}}}^{}\phi_{i}^{\text{x}}\left((\boldsymbol{r}^{\text{orig}}_i,\boldsymbol{r}^{\text{new}}_i)\right)\\
	&\overset{\text{Lapl.}}{=} \frac{1}{Z^{\text{x}}}\sum_{\boldsymbol{r}^{\text{new}}\in\mathcal{X}_{\boldsymbol{R}^{\text{new}}}}\prod_{\phi_{i}\in \boldsymbol{\Phi}^{\text{orig}}}\phi_{i}(\boldsymbol{r}^{\text{orig}}_i)\cdot c_i\cdot \hspace{-0.2cm}\prod_{\phi_{j}^{\text{x}}\in \boldsymbol{\Phi}^{\text{new}}}^{}\hspace*{-0.5em} c_j\\
    &\overset{\text{Th.~\ref{theorem:generalpartitionfunction}}}{=}\frac{\sum_{\boldsymbol{r}^{\text{new}}\in\mathcal{X}_{\boldsymbol{R}^{\text{new}}}}1}{Z\cdot \vert \mathcal{X}_{\boldsymbol{R}^{\text{new}}}\vert} \prod_{\phi_{i}\in \boldsymbol{\Phi}^{\text{orig}}}\phi_{i}(\boldsymbol{r}^{\text{orig}}_i)\\
    &\overset{\text{def.}}{=} \frac{\vert \mathcal{X}_{\boldsymbol{R}^{\text{new}}}\vert}{\vert \mathcal{X}_{\boldsymbol{R}^{\text{new}}}\vert}\;\cdot P_M\left(\boldsymbol{r}^{\text{orig}}\right)=P_M\left(\boldsymbol{r}^{\text{orig}}\right).  \tag*{\qedsymbol} 
\end{align*}
\renewcommand{\qedsymbol}{}
\end{proof}
\vspace{-0.5cm}
If an \ac{fg} is extended solely by additional Laplace factors, or by \acp{lx} involving only \acp{rv} already present in the underlying \ac{ms}, then the probability space itself remains unchanged.
In contrast, the introduction of a new \ac{rv} necessarily enlarges the \ac{ms} via the Cartesian product of the original space with the domain of the new variable, and the associated $\sigma$-algebra expands accordingly as its power set.

Nevertheless, the resulting probability measure is pro\-jection-preserving regarding the original (smaller) probability space: for every event in $\mathcal{P}(\mathcal{X}_{\boldsymbol{R}^{\text{orig}}})$ there exists a corresponding preimage in $\mathcal{P}(\mathcal{X}_{\boldsymbol{R}^{\text{x}}})$.
Hence, every measurable set of the original space has an origin in the extended space, and the mapping induced by marginalisation over $\boldsymbol{R}^{\text{new}}$ is surjective.

\section{Comparability of Factor Graphs}\label{section_four}
This section systematically works towards the comparability of two arbitrary \ac{fg} that have initially been defined over different \acp{ms}, by constructing a minimal number of \acp{lx} to lift the \acp{fg} to the same \ac{ms} while aligning their graphical structure, called \ac{msx}.
An extension allows for comparability from a theoretical viewpoint for common distance measures and divergence measures for probability distributions, but also allows due to the same graphical structure to compare them on factor-level, which is especially interesting for \acp{pgm}.

\subsection{Factor Structures}
To identify common structure between factor graphs, we introduce four notions that characterise how individual factors relate across graphs and within extensions.

\begin{definition}\label{def:struciden}
    Factors $\phi_i$ of \ac{fg} $M$ and $\phi'_j$ of \ac{fg} $M'$ are \textbf{structurally identical} if and only if $\boldsymbol{R}^{\text{orig}}_{(i)}=\boldsymbol{R'}^{\text{orig}}_{(j)}$.
    Two \acp{fg} $M$ and $M'$ are \textbf{structurally identical} if and only if there exists a bijection $\beta \colon \boldsymbol{\Phi}^{\text{orig}}_{M} \to \boldsymbol{\Phi}^{\text{orig}}_{M'}$ from factors $\phi_i \in \boldsymbol{\Phi}^{\text{orig}}_{M}$ in $M$ to factors $\phi'_j \in \boldsymbol{\Phi}^{\text{orig}}_{M'}$ in $M'$ such that $\boldsymbol{R}^{\text{orig}}_{(i)}=\boldsymbol{R'}^{\text{orig}}_{(j)}$ holds.  
\end{definition}

Structural identity captures exact agreement of factor scopes, the opposite can be described as independence, formalising the absence of shared variables and allowing for extensions that do not interact with the original factorisation.

\begin{definition}
	Two factors $\phi_i$ and $\phi_j$ with $i\neq j$ 
    are called \textbf{independent} if and only if $\boldsymbol{R}^{\text{orig}}_{(i)} \cap \boldsymbol{R}^{\text{orig}}_{(j)} =\emptyset$.
    Let $M^{\text{x}}$ be an extension of $M$ and let $\phi_j^{\text{x}}\in\boldsymbol{\Phi}^{\text{new}}$.
    The factor $\phi_j^{\text{x}}\in\boldsymbol{\Phi}^{\text{new}}$ is called an \textbf{independent extension} of $M$ if it is independent of all extended factors $\phi_i^{\text{x}}\in\boldsymbol{\Phi}^{\text{orig},\text{x}}$ of the original \ac{fg}~$M$.
\end{definition}

\begin{lemma}
    Let $M^{\text{x}}$ be a factor graph extension of $M$.
    If every factor $\phi^{\text{x}}_i\in \boldsymbol{\Phi}^{\text{new}}$ is an independent extension of $M$ and $\boldsymbol{\Phi}^{\text{orig},\text{x}}=\boldsymbol{\Phi}^{\text{orig}}$, then the partition function of the extended \ac{fg} $M^{\text{x}}$ is given by
	\begin{align*}
		Z^{\text{x}} =Z \cdot  \sum_{\boldsymbol{r}^{\text{new}}\in\mathcal{X}_{\boldsymbol{R}^{\text{new}}}}\prod_{\phi_i^{\text{x}}\in\boldsymbol{\Phi}^{\text{new}}} \phi_i^{\text{x}}(\boldsymbol{r}_i^{\text{new}}).
	\end{align*}
\end{lemma}
\begin{proof}[Proof Sketch]
 Due to normalisation of probability measures, we get $1= \sum_{\boldsymbol{r}^{\text{x}}\in \mathcal{X}_{\boldsymbol{R}^{\text{x}}}}P_{M^{\text{x}}}(\boldsymbol{r}^{\text{x}})$.
 Using independence of the different sets, rearranging the terms yields $1=Z/Z^{\text{x}} \cdot\sum_{\boldsymbol{r}^{\text{new}}\in \mathcal{X}_{\boldsymbol{R}^{\text{new}}}}\prod_{\phi_i^{\text{x}}\in \boldsymbol{\Phi}^{\text{new}}}\phi_i^{\text{x}}(\boldsymbol{r}^{\text{new}}_i)$.
\end{proof}

For scope inclusion, we define the following strictly hierarchical relation between factors.
\begin{definition}
	A factor $\phi_{i}$ of \ac{fg} $M$ is \textbf{subsumed} in another factor $\phi'_j$ of \ac{fg} $M'$ if and only if $\boldsymbol{R}^{\text{orig}}_{(i)}\subsetneq \boldsymbol{R'}_{\kern-0.4em (j)}^{\text{orig}}$.
\end{definition}
Overlapping factors constitute the most complex case, as they are neither independent nor related by subsumption.
\begin{definition}
	A factor $\phi_{i}$ of \ac{fg} $M$ and factor $\phi'_j$ of \ac{fg} $M'$ are \textbf{overlapping factors} if and only if they have partially overlapping scopes, that is, if
    $\boldsymbol{R}_{(i)}^{\text{orig}}\cap \boldsymbol{R'}^{\text{orig}}_{\kern-0.4em (j)}\neq \emptyset$ with $\boldsymbol{R}^{\text{orig}}_{(i)}\setminus\boldsymbol{R'}^{\text{orig}}_{\kern-0.4em (j)}\neq \emptyset$ and $\boldsymbol{R'}^{\text{orig}}_{\kern-0.4em (j)}\setminus \boldsymbol{R}^{\text{orig}}_{(i)}\neq \emptyset$ holds.
\end{definition}

\subsection{Minimal Structural Laplace Extension}
Without loss of generality, we assume that we have a unique factor representation for a given \ac{fg} $M$, in which structurally identical as well as subsumed factors are merged into a single factor.
Formally, for any $\phi_i, \phi_j \in \boldsymbol{\Phi}^{\text{orig}}$, we have
\begin{align*}
    \boldsymbol{R}^{\text{orig}}_{(i)} \neq \boldsymbol{R}^{\text{orig}}_{(j)} \; \text{and} \;
    \boldsymbol{R}^{\text{orig}}_{(i)} \not\subsetneq \boldsymbol{R}^{\text{orig}}_{(j)}.
\end{align*}
Overlapping factors are allowed and considered part of the unique representation.
 \begin{algorithm}[tb]
    \caption{MSLX (Min. Struct. Lapl. Ext.)}\label{algo:msx}
     \textbf{Input}: Two \acp{fg}: $M$ and $M'$\\
     \textbf{Output}: Two extended \acp{fg} $M^{\text{x}}$ and $M'^{\text{x}}$

    \begin{algorithmic}[1] 
        \State Extended sets of \acp{rv} $\boldsymbol{R}^{\text{x}},\boldsymbol{R'}^{\text{x}} \leftarrow \boldsymbol{R}^{\text{orig}} \cup \boldsymbol{R'}^{\text{orig}}$
        \State Initialise extended factor sets\\
            $\;\boldsymbol{\Phi}^{\text{orig},\text{x}},\boldsymbol{\Phi}^{\text{new}},\boldsymbol{\Phi'}^{\text{orig},\text{x}},\boldsymbol{\Phi'}^{\text{new}} \leftarrow \{\}$
        \State Initialise index sets\\
        $\; I_M \leftarrow \{1,\ldots,\lvert\boldsymbol{\Phi}^{\text{orig}} \rvert\}$ and $I_{M'} \leftarrow \{1,\ldots,\lvert\boldsymbol{\Phi'}^{\text{orig}} \rvert\}$
        \For{$i \in I_M$, $j \in I_{M'}$}
            \If{$\phi_i$ is \textit{independent} of all $\phi'_j$}
                \State Add $\phi_i$ as \ac{lf} to $\boldsymbol{\Phi'}^{\text{new}}$
                \State Remove $i$ from $I_M$ 
            \EndIf 
            \If{$\phi'_j$ is \textit{independent} of all $\phi_i$}
                \State Add $\phi'_j$ as \ac{lf} to $\boldsymbol{\Phi}^{\text{new}}$
                \State Remove $j$ from $I_{M'}$
            \EndIf
            \If{$\phi_i$ and $\phi_j$ are \textit{structurally identical}}
                \State Add $\phi_i$ to $\boldsymbol{\Phi}^{\text{orig},\text{x}}$, $\phi'_j$ to $\boldsymbol{\Phi'}^{\text{orig},\text{x}}$
                \State Remove $i$ from $I_M$, $j$ from $I_{M'}$ 
            \EndIf
            \If{$\phi_i$ is subsumed in $\phi'_j$}
                \State Add \ac{lx} of $\phi_i$ over $\boldsymbol{R'}^{\text{orig}}_{\kern-0.4em(j)}$ to $\boldsymbol{\Phi}^{\text{orig},\text{x}}$, $\phi'_j$ to $\boldsymbol{\Phi'}^{\text{orig},\text{x}}$
                \State Remove $i$ from $I_M$ and $j$ from $I_{M'}$
            \EndIf
            \If{$\phi'_j$ is subsumed in $\phi_i$}
                \State Add $\phi_i$ to $\boldsymbol{\Phi}^{\text{orig},\text{x}}$, \ac{lx} of $\phi'_j$ over $\boldsymbol{R}^{\text{orig}}_{(i)}$ to $\boldsymbol{\Phi'}^{\text{orig},\text{x}}$
                \State Remove $i$ from $I_M$, $j$ from $I_{M'}$
            \EndIf
            \If{$\phi_i$ and $\phi'_j$ are \textit{overlapping}}
                \State Add \ac{lx} of $\phi_i$ over $\boldsymbol{R}^{\text{orig}}_{(i)}\cup \boldsymbol{R'}^{\text{orig}}_{(j)}$ to $\boldsymbol{\Phi}^{\text{orig},\text{x}}$,\newline
                \hspace*{1.5em}$\phi'^{\text{x}}_j$ as \ac{lx} of $\phi'_j$ over $\boldsymbol{R}^{\text{orig}}_{(i)}\cup \boldsymbol{R'}^{\text{orig}}_{(j)}$ to $\boldsymbol{\Phi'}^{\text{orig},\text{x}}$
            \EndIf
        \EndFor
        \State Construct $\boldsymbol{E}^{\text{x}}$ and $\boldsymbol{E'}^{\text{x}}$ induced by scopes of factors
    \end{algorithmic}
\end{algorithm}

\Ac{msx}, shown in \cref{algo:msx}, constructs \acp{lx} for two \acp{fg}, aligning them on the same graph structure and \ac{ms}.
\cref{algo:msx} constructs an extension on the \textit{smallest} common \ac{ms} by adding the \textit{minimal} set of variables required to achieve structural equality of corresponding factors across both graphs. 
Formally, an extension is minimal if no strictly smaller Laplace extension (measured in terms of added variables and induced scopes) yields the same aligned structure.
Each factor is extended at most once and only when necessary, i.e., when no counterpart with identical scope exists in the other graph.
While arbitrary large Laplace extensions are always possible, \cref{algo:msx} avoids any non-essential augmentation by construction.
Minimality is thus not defined in terms of the number of factors, but as the minimal structural completion required for consistent local comparison, using only Laplace extensions and without merging factors.
Any omission would prevent alignment, whereas any additional variable would strictly enlarge scopes without improving comparability, thereby violating minimality.

By construction, we get the following result.
\begin{theorem}\label{theorem:structurallyidentical}
The \acp{fg} $M^{\text{x}}$ and $M'^{\text{x}}$ returned by \cref{algo:msx} for two input \acp{fg} $M$ and $M'$ are structurally identical and encode probability distributions on the same \ac{ms}
\begin{align*}
  \left(\mathcal{X}_{\boldsymbol{R}^{\text{orig}}\;\cup\;\boldsymbol{R'}^{\text{orig}}}, \mathcal{P}\left(\mathcal{X}_{\boldsymbol{R}^{\text{orig}}\;\cup\;\boldsymbol{R'}^{\text{orig}}}\right)\right).
\end{align*}
\end{theorem}

\begin{proof}
For every individual factor $\phi_i\in\boldsymbol{\Phi}^{\text{orig}}$ or $\phi'_j\in\boldsymbol{\Phi'}^{\text{orig}}$, there is an Laplace extension step within \cref{algo:msx}, which means that the outcomes $M^{\text{x}}$ and $M'^{\text{x}}$ are indeed extended \acp{fg}.
By construction, \cref{algo:msx} enforces $\boldsymbol{R}^{\text{x}}=\boldsymbol{R}^{\text{orig}}\cup\boldsymbol{R'}^{\text{orig}}=\boldsymbol{R'}^{\text{x}}$.




Every individual extension step applied to $M$ is a \ac{lx} by construction, which makes \cref{theorem:generalpartitionfunction} applicable.
Therefore, any single assignment $\boldsymbol{r}^{\text{x}}\in \boldsymbol{R}^{\text{x}}$ is uniquely defined via
\begin{align*}
    P_{M^{\text{x}}}(\boldsymbol{r}^{\text{x}}) & = \frac{1}{Z^{\text{x}}} \prod_{\phi_i^{\text{x}}\in \boldsymbol{\Phi}^{\text{x}}} \phi_i^{\text{x}}(\boldsymbol{r}_i^{\text{x}}) 
\end{align*}
with corresponding constants. 
Therefore, $P_{M^{\text{x}}}$ defines a probability measure on the \ac{ms} $(\mathcal{X}_{\boldsymbol{R}^{\text{orig}}\;\cup\;\boldsymbol{R'}^{\text{orig}}}, \mathcal{P}(\mathcal{X}_{\boldsymbol{R}^{\text{orig}}\;\cup\;\boldsymbol{R'}^{\text{orig}}}))$.

It remains to show structural identity by constructing a bijection $\beta \colon \boldsymbol{\Phi}^{\text{x}}\rightarrow\boldsymbol{\Phi'}^{\text{x}}$ as required in \cref{def:struciden}.
Recall that we assume that neither of the original \acp{fg} $M$ nor $M'$ contain structurally identical factors within themselves.
For each original factor $\phi_i\in\boldsymbol\Phi^{\text{orig}}$,
the mapping is defined as follows:\vspace{-0.25cm}
\begin{itemize}
    \item If $\phi_i$ and some $\phi'_j\in\boldsymbol\Phi'^{\text{orig}}$
    are structurally identical or one subsumes the other, then
    $\beta(\phi_i^{\text{x}})=\phi_j'^{\text{x}}$.\vspace{-0.1cm}
    \item If $\phi_i$ is independent of all factors in $M'$,
    then an \ac{lf} $\phi_j'^{\text{x}}\in\boldsymbol\Phi'^{\text{new}}$
    with identical scope is introduced and
    $\beta(\phi_i^{\text{x}})=\phi_j'^{\text{x}}$.\vspace{-0.1cm}
    The symmetric case is handled analogously.
    \item For every pair of overlapping factors $\phi_i$ and $\phi'_j \in \boldsymbol\Phi'^{\text{orig}}$, \cref{algo:msx} introduces \acp{lx} of both factors
    to the united scope, and
    $\beta(\phi_i^{\text{x}})=\phi_j'^{\text{x}}$.\vspace{-0.1cm}
\end{itemize}

By construction, all matched factor pairs share the same scope
$\boldsymbol R^{\text{x}}_{(i)}=\boldsymbol R'^{\text{x}}_{(j)}$.
Hence, $\beta$ is well-defined and bijective, and
$M^{\text{x}}$ and $M'^{\text{x}}$ are struct. identical.
\end{proof}

\paragraph{Remark:}
In the final \textbf{if}-statement of \cref{algo:msx}, the indexes are intentionally retained, as multiple factors from one \ac{fg} may overlap with the same factor of the other graph.
This induces the same scopes for these factors in the extended \ac{fg}. 
In principle, such factors could be merged together, as assumed for the original \ac{fg} at the beginning of this section. 
However, merging might yield a loss of information regarding which parts are behaving as Laplace influence, because the merged factor is not necessarily Laplace in the same \ac{rv}, which violates the construction of being an \ac{lx}.

Keeping the previous remark in mind, we can even reverse the process of an \ac{msx}.
\paragraph{Returning to the Original Factor Graph:} 
In principle, \acp{lx} can be reversed by inspecting an \ac{fg} for trivial factors or for \acp{rv} with constant influence within a factor.
The associated constants can be absorbed into the partition function.
However, if the original \ac{fg} already contained a factor with a \ac{rv} of constant influence (Laplace), reversing the extension could yield an even smaller \ac{fg} as the original \ac{fg} is not retained during the extension.
While this direction is not explored in the present work, it is always possible to identify a ``minimal'' \ac{fg}, which contains no \acp{lf} and no \acp{rv} with constant influence within any factor.
The minimal \ac{fg} can be regarded as a representative of an equivalence class of \acp{fg}, comprising all \acp{lx} of this minimal structure.
In contrast to \textit{variable elimination} \citep{ZhaPo94}, which eliminates variables from the whole \ac{fg}, this procedure allows the elimination of individual \acp{rv} from a factor even if they remain present in other factors.
Thus, in case the original \ac{fg} is minimal, i.e., it has no Laplace \acp{rv}, the extensions can be reversed again and the original \ac{fg} can be obtained.

\subsection{Comparison of Two Factor Graphs}\label{subsection:four}
Given two \acp{fg} $M$ and $M'$, 
defined on different finite \acp{ms} and involving arbitrarily different sets of \acp{rv}, a direct comparison is in general not meaningful.
By applying \cref{algo:msx} and 
\cref{theorem:structurallyidentical}, both \acp{fg} are lifted to \acp{lx} $M^{\text{x}}$ and $M'^{\text{x}}$ that share the same graphical structure and \ac{ms}.

By \cref{theorem:surjection}, each extended \ac{fg} obtains a measure-preserving projection back to its original distribution, preserving probabilistic behaviour on the original space $\mathcal{X}_{\boldsymbol{R}^{\text{orig}}}$ and $\mathcal{X}_{\boldsymbol{R'}^{\text{orig}}}$, respectively.

At the level of the full joint distribution, standard distances, metrics, and divergence measures, e.g. \textit{Kullback-Leibler Divergence}~\citep{kullback1951information}, \textit{total variation distance}~\citep{Bretagnolle1978estimation}, \textit{Wasserstein metric}~\citep{vaserstein1969markov}, \textit{Hellinger distance}~\citep{hellinger1909neue} and more, can be applied to the extended \acp{fg}, see \cref{fig:metrics}.
However, preserving the original factorisation while extending the graph enables comparisons at the local / conditional factor level rather than only globally.

One local measure has been proposed by \cite{Chan2005a}, who focus on relative differences between the extrema of factor potentials.
Our Laplace-based introduced construction might be more suitable for deviation measures for an expected uniform influence, emphasising the importance of a weighting procedure for factors (\cref{corollary:lfg}).

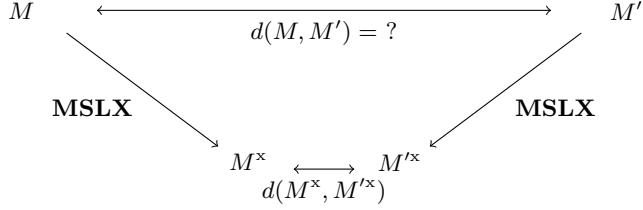
\begin{figure}[t]
\centering
\begin{tikzpicture}[
    factor/.style={rectangle, draw, fill=black, minimum size=3pt, inner sep=0pt},
    var/.style={circle, draw, fill=white, minimum size=3pt, inner sep=0pt},
    every node/.style={font=\small}
]

\node (M) at (-4,2) {$M$};

\node (Mp) at (4,2) {$M'$};

\draw[<->] (-3,2) -- (3,2)
    node[midway, below] {$d(M,M') =\ ?$};

\draw[->] (-3.4,1.7) -- (-1.4,0.2)
    node[midway, below left] {\textbf{MSLX}};
\draw[->] (3.4,1.7) -- (1.4,0.2)
    node[midway, below right] {\textbf{MSLX}};

\node (Mx) at (-1,0) {$M^{\mathrm{x}}$};

\node (Mpx) at (1,0) {$M'^{\mathrm{x}}$};

\draw[<->] (-0.4,-0.1) -- (0.4,-0.1)
    node[midway, below] {$d({M^{\text{x}}},{M'^{\text{x}}})$};
\end{tikzpicture}\vspace{-0.25cm}
\caption{Schematic illustration of factor graphs $M$ and $M'$ and their Laplace extensions $M^{\text{x}}$ and $M'^{\text{x}}$ via \cref{algo:msx}.
A distance measure $d$ can only be applied between the extensions, which are defined on the same \ac{ms}.}\label{fig:metrics}
\end{figure}

The identical structure also enables a structural analysis.
In particular, structural differences are encoded exclusively by Laplace components, while all remaining discrepancies are captured by the non-Laplace potentials.
Even if two \acp{fg} are already defined on the same \ac{ms}, applying the extension remains meaningful, as it provides the additional option of local analysis.

In the cases where only one \ac{fg} is given, a comparison to the structurally identical Laplace \ac{fg} is possible (\cref{corollary:lfg}), where deviations can be quantified as the relative to a fully Laplace \ac{fg} baseline (see~\cref{append:BExample}). 

\section{Discussion}
\paragraph{Comparison to Related Work:}
Comparing structured probabilistic models across different dimensions, \acp{ms}, or structural assumptions is a well-known challenge.
\citet{cai2022distances} aim to define meaningful distances under such heterogeneity.
\ac{msx} differs in that it explicitly constructs a minimal structural extension of an \ac{fg}, providing a representation that preserves probabilistic inference results.
\acp{fg} are particularly amenable to uniform (Laplace) extensions, which preserve the original distribution while aligning scopes across models.
Further, in principle, every discrete distribution could be interpreted as a deviation from a uniform reference, with dependencies encoded via Laplace factors.\vspace{-0.25cm}

\paragraph{Opportunities and Advantages:}
A key benefit of \ac{msx} is that one does not have to  explicitly compute full joint distributions or partition functions.
By normalising factors as weighted influences on an underlying locally independent uniform reference, \ac{msx} achieves structural comparability with minimal computational overhead.
\Ac{msx} is fully canonical up to isomorphism, enabling repeatable and interpretable comparisons.
Furthermore, by aligning factor scopes, local relative similarities become computable, supporting fine-grained, stepwise analyses of structural changes.\vspace{-0.25cm}

\paragraph{Challenges and Limitations:}
Some assumptions require further consideration.
First, uniqueness of factor scopes is essential for unambiguous interpretation. 
While the original \acp{fg} typically satisfy uniqueness of factor scopes, extensions might introduce redundant scopes that may need merging, which, however, leads to a loss of structural comparisons to the original \ac{fg}.
Second, while \acp{lx} are measure-preserving projections, semantic differences between original factors are disregarded, which may limit interpretability in some applications and need further investigation.
Third, although \ac{msx} provides a structural alignment, quantitative measures of distance or impact on the distribution (e.g., partition function) require further formalisation, especially when considering local (non-uniform) changes of potentials.\vspace{-0.25cm}

\paragraph{Directions for Future Work:}
The \ac{msx} framework opens multiple avenues for future exploration:\vspace{-0.25cm}
\begin{itemize}
    \item \textbf{Local similarity and explainability:} Investigate stepwise changes in individual factors or subgraphs, estimating how small modifications spread through the model and affect other parts of the distribution.
    This could lead to transparent feedback mechanisms for structural change or dependency evaluation.  \vspace{-0.1cm}
    \item \textbf{Hierarchical and relational models:} Many structured probabilistic models naturally extend one another.
    Understanding how minimal extensions operate in hierarchies could enable efficient comparisons and partial alignments \citep{Speller2025a}.  \vspace{-0.1cm}
    \item \textbf{Factor-level measures:} Develop metrics explicitly targeting differences between factors, leveraging the \ac{lx} mechanism to transform factors into a reference representation on a common \ac{ms}.  \vspace{-0.1cm}
    \item \textbf{Extension to other factorised models:} Examine if the introduced concept of minimal Laplace extension generalises to other structured (factorised) probabilistic representations beyond \acp{fg}.  \vspace{-0.1cm}
\end{itemize}

\section{Conclusion}
This work introduces \ac{msx} as a principled framework for inducing comparability between arbitrary \acp{fg} via minimal structural extension. 
Based on the idea of uniform influence, we present a deterministic construction procedure to find a surpassing graphical representation, while preserving probabilistic consistency. 
\ac{msx} establishes a sound foundation for comparing, aligning, and aggregating factorised probability distributions that are originally incomparable.
Beyond global comparability of the full graph, \ac{msx} constitutes a versatile tool to generate common reference representations for structured probabilistic models and creates new opportunities to investigate local differences for subgraphs, factors, or local neighbourhoods.
Moreover, it supports the definition of similarity measures at finer levels of granularity beyond the full joint distribution and across hierarchical or stepwise differences, yielding a more transparent feedback mechanism for structural changes and dependencies.

\subsubsection*{Acknowledgements}
This work was partially funded by the Ministry of Culture and Science of the German State of North Rhine-Westphalia.



\bibliographystyle{plainnat}
\bibliography{References} 

\clearpage
\appendix
\thispagestyle{empty}

\onecolumn
\aistatstitle{Supplementary Materials} 
\section{Example MSLX}
\begin{example}
Let $M = (\boldsymbol{V}, \boldsymbol{E})= (\boldsymbol{R} \cup \boldsymbol{\Phi}, \boldsymbol E)= ( \{X_2, \ldots, X_6\} \cup \{\phi_1, \ldots, \phi_4\},\boldsymbol{E})$ and $M'=(\{X_1,X_3, \ldots, X_6\} \cup \{\phi'_1, \ldots, \phi'_3\},\boldsymbol{E}')$ be two \acp{fg} with corresponding edges $\boldsymbol{E}$ and $\boldsymbol{E}'$, respectively (see~\cref{fig:exampleappendix}).
Let $M^{\text{x}}$ and $M'^{\,\text{x}}$ be its extensions and outcomes of \cref{algo:msx}, then they have the same $\boldsymbol{R}^{\text{x}}=\{X_1, \ldots, X_6\}=\boldsymbol{R}'^{\text{x}}$ with identical argument lists $\mathcal{X}_{\boldsymbol{R}^{\text{x}}_{(i)}}=\mathcal{X}_{\boldsymbol{R}'^{\text{x}}_{(i)}}$ for $i=1,\ldots,4$ for their factors $\phi_i$ and $\phi'_i$ and the same \ac{ms} $(\mathcal{X}_{\boldsymbol{R}^{\text{x}}},\mathcal{P}(\mathcal{X}_{\boldsymbol{R}^{\text{x}}}))=(\mathcal{X}_{\boldsymbol{R}'^{\text{x}}},\mathcal{P}(\mathcal{X}_{\boldsymbol{R}'^{\text{x}}}))$.
\begin{figure}[!htbp]
	\begin{center}\hspace*{-2.1cm}
		\begin{tikzpicture}[>=stealth]
			factor/.style={rectangle, draw, fill=black, minimum size=3pt, inner sep=0pt},
			var/.style={circle, draw, fill=white, minimum size=3pt, inner sep=0pt},
			every node/.style={font=\small}
			]
			
			
			
			
                    \node (M) at (-5.55,5) {$M$};
				
					\node (Mp) at (10.42,5) {$M'$};
				\node (OL) at (-2.6,4) {
					\begin{minipage}{0.59\textwidth}
						\centering
						\scalebox{0.75}{\begin{tikzpicture}
	\node[rv, opacity=0, draw = blue] (ceins) {$X_1$};
	\node[rv, right = 1.8cm of ceins] (czwei) {$X_2$};
	\node[rv, right = 1.8cm of czwei] (cdrei) {$X_3$};
	
	\factor{below}{ceins}{0.9cm}{180}{$\phi_1$}{f1}
	
	\factor{below}{czwei}{0.9cm}{270}{$\phi_2$}{f2}
	\factor{below}{cdrei}{0.9cm}{0}{$\phi_3$}{f3}

	\node[rv, right = 0.9cm of f1] (cvier) {$X_4$};
	\node[rv, right = 0.9cm of f2] (cfunf) {$X_5$};
	
	\node[rv, below = 0.9cm of cvier] (csechs) {$X_6$};
	\factor{right}{csechs}{0.9cm}{0}{$\phi_4$}{f4}
	
	\draw (f1) -- (cvier);
	\draw (f1) -- (csechs);
	
	\draw (f2) -- (czwei);
	\draw (f2) -- (cvier);
	\draw (f2) -- (cfunf);
	
	\draw (f3) -- (cfunf);
	\draw (f3) -- (cdrei);
	
	\draw (f4) -- (csechs);
	\draw (f4) -- (cfunf);
    \node[above=20mm] at (f2) {$\left(\Omega_{M} = \times_{i=2}^{6}\mathcal{X}_{X_i}, \sigma_{\text{pow}}(\Omega_{M})\right)$};
\end{tikzpicture}}
					\end{minipage}
				};
				
				\node (OR) at (7.6,4) {
					\begin{minipage}{0.35\textwidth}
						\centering
						\scalebox{0.75}{\begin{tikzpicture}
	\node[rv] (ceins) {$X_1$};
	\node[rv, opacity=0, right = 1.8cm of ceins, draw = blue] (czwei) {$X_2$};
	\node[rv, right = 1.8cm of czwei] (cdrei) {$X_3$};
	
	\factor{below}{ceins}{0.9cm}{180}{$\phi'_1$}{f1}
	
	\factor{below}{czwei}{0.9cm}{270}{$\phi'_2$}{f2}
	\factor{below}{cdrei}{0.9cm}{0}{$\phi'_3$}{f3}

	\node[rv, right = 0.9cm of f1] (cvier) {$X_4$};
	\node[rv, right = 0.9cm of f2] (cfunf) {$X_5$};
	
	\node[rv, below = 0.9cm of cvier] (csechs) {$X_6$};
	\node[rv, opacity=0, right = 0.9cm of csechs] (f4) {$\phi'_4$};
	
	\draw (f1) -- (ceins);
	\draw (f1) -- (cvier);
	\draw (f1) -- (csechs);
	
	\draw (f2) -- (cvier);
	\draw (f2) -- (cfunf);
	
	\draw (f3) -- (cfunf);
	\draw (f3) -- (cdrei);

     \node[above=20mm] at (f2) {$\left(\Omega_{M'} =\mathcal{X}_{X_1} \times \left(\times_{i=3}^{6}\mathcal{X}_{X_i}\right), \sigma_{\text{pow}}(\Omega_{M'})\right)$};
     
\end{tikzpicture}}
					\end{minipage}
				};
				
				\node (UL) at (-2.6,-.50) {
					\begin{minipage}{0.59\textwidth}
						\centering
						\scalebox{0.75}{\begin{tikzpicture}
	\node[rv, draw = blue, text=blue] (ceins) {$X_1$}; 
	\node[rv, right = 1.8cm of ceins] (czwei) {$X_2$};
	\node[rv, right = 1.8cm of czwei] (cdrei) {$X_3$};
	
	\factor{below}{ceins}{0.9cm}{180}{$\phi^{\textcolor{blue}{\,\text{x}}}_1$}{f1}
	\factorcolored{below}{ceins}{0.9cm}{180}{$\phi^{\textcolor{blue}{\,\text{x}}}_1$}{f1}{blue}
	\path (f1) node[draw=blue, fill=gray!30] {};
	
	\factor{below}{czwei}{0.9cm}{270}{$\phi_2$}{f2}
	\factor{below}{cdrei}{0.9cm}{0}{$\phi_3$}{f3}

	\node[rv, right = 0.9cm of f1] (cvier) {$X_4$};
	\node[rv, right = 0.9cm of f2] (cfunf) {$X_5$};
	
	\node[rv, below = 0.9cm of cvier] (csechs) {$X_6$};
	\factor{right}{csechs}{0.9cm}{0}{$\phi_4$}{f4}

	\draw[blue, dashed] (f1) -- (ceins);   
	\draw (f1) -- (cvier);
	\draw (f1) -- (csechs);
	
	\draw (f2) -- (czwei);
	\draw (f2) -- (cvier);
	\draw (f2) -- (cfunf);
	
	\draw (f3) -- (cfunf);
	\draw (f3) -- (cdrei);
	
	\draw (f4) -- (csechs);
	\draw (f4) -- (cfunf);
   \node[below=20mm] at (f2) {$\left(\Omega_{M^{\textcolor{blue}{\text{x}}}} =\textcolor{blue}{\mathcal{X}}_{\textcolor{blue}{X_1}}\times\left(\times_{i=2}^{6}\mathcal{X}_{X_i}\right) , \sigma_{\text{pow}}(\Omega_{M^{\textcolor{blue}{\text{x}}}})\right)$};
\end{tikzpicture} 
}
					\end{minipage}
				};
				
				\node (UR) at (7.6,-.50) {
					\begin{minipage}{0.35\textwidth}
						\centering
						\scalebox{0.75}{\begin{tikzpicture}
	\node[rv] (ceins) {$X_1$};
	\node[rv, right = 1.8cm of ceins, draw = blue, text=blue] (czwei) {$X_2$}; 
	\node[rv, right = 1.8cm of czwei] (cdrei) {$X_3$};
	
	\factor{below}{ceins}{0.9cm}{180}{$\phi'_1$}{f1};
	
	\factorcolored{below}{czwei}{0.9cm}{270}{$\phi_2'^{\textcolor{blue}{\,\text{x}}}$}{f2}{blue} 
	\path (f2) node[draw=blue, fill=gray!30] {};
	
	\factor{below}{cdrei}{0.9cm}{0}{$\phi_3'$}{f3}

	\node[rv, right = 0.9cm of f1] (cvier) {$X_4$};
	\node[rv, right = 0.9cm of f2] (cfunf) {$X_5$};
	
	\node[rv, below = 0.9cm of cvier] (csechs) {$X_6$};
	\factorcolored{right}{csechs}{0.9cm}{0}{$\color{blue}{\phi_4'^{\,\text{x}}}$}{f4}{blue}
	
	\draw (f1) -- (ceins);
	\draw (f1) -- (cvier);
	\draw (f1) -- (csechs);
	
	\draw[blue, dashed] (f2) -- (czwei);   
	\draw (f2) -- (cvier);
	\draw (f2) -- (cfunf);
	
	\draw (f3) -- (cfunf);
	\draw (f3) -- (cdrei);
	
	\draw[blue, dashed] (f4) -- (csechs);
	\draw[blue, dashed] (f4) -- (cfunf);
      \node[below=20mm] at (f2) {$\left(\Omega_{M'^{\,\textcolor{blue}{\text{x}}}} =  \mathcal{X}_{X_1}\times \textcolor{blue}{\mathcal{X}}_{\textcolor{blue}{X_2}}\times\left(\times_{i=3}^{6}\mathcal{X}_{X_i}\right), \sigma_{\text{pow}}(\Omega_{M'^{\,\textcolor{blue}{\text{x}}}})\right)$};
      
\end{tikzpicture}}
					\end{minipage}
				};
				
				
                \draw[<->, thick,color=black]
  ([xshift=-14mm]UL.east) --
  node[midway, above] {$d(M^{\textcolor{blue}{\text{x}}},M'^{\,\textcolor{blue}{\text{x}}})$}
  node[midway, yshift=-4mm] {\textcolor{blue}{same} \ac{ms}}
  ([xshift=-6mm]UR.west);
    \draw[<->, thick,color=black]
  ([xshift=-14mm]OL.east) --
  node[below] {$d(M,{M'})=?$}
  node[midway,yshift=4mm] {different \acp{ms}}
  ([xshift=-6mm]OR.west);
                \draw[->, thick,color=black] (OL) -- node[left]{MSLX} (UL);
				\draw[->, thick,color=black] (OR) -- node[right]{MSLX} (UR);
				\node (MX) at (-5.55,-0.5) {$M^{\textcolor{blue}{\text{x}}}$};
				\node (MpX) at (10.42,-0.5) {$M'^{\,\textcolor{blue}{\text{x}}}$};
                        
		\end{tikzpicture}
	\end{center}
      \caption{Two \acp{fg} $M$ and $M'$ defined on different \acp{ms} and therefore different factors get extended via \cref{algo:msx} to achieve structural equality on the graph level while the probability distributions are defined on the same \ac{ms}.}
    \label{fig:exampleappendix}
\end{figure}
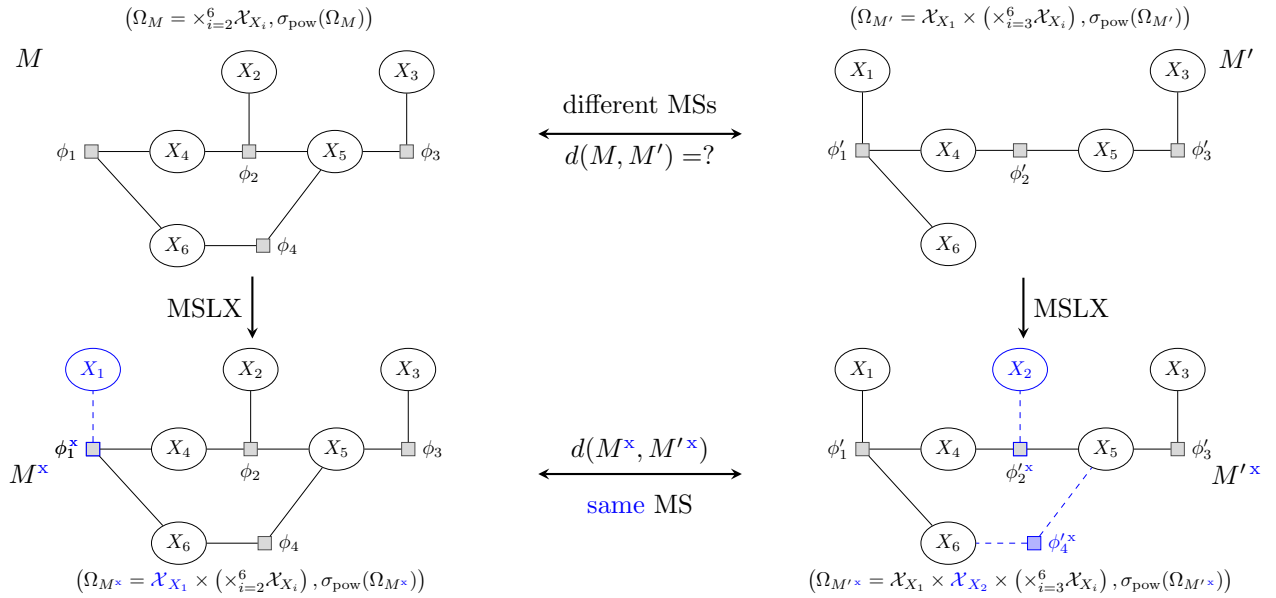
\end{example}

\section{Uniqueness and Order-invariance as a Consequence of Minimality}
For any pair of input factor graphs, MSLX produces a unique and stable output structure (up to normalisation), as the construction is fully determined by minimal scope alignment.
Hence, the resulting extended factor graph is independent of algorithmic ordering of factors or representation choices.
This is due to the order of the algorithmic case distinctions (independent, overlapping etc.) and the allowance of multiple factors with the same scope after extension.
Post-processing steps such as merging factors with identical scopes are separate from the construction.
In particular, MSLX preserves the provenance of factors, i.e., it retains the mapping from each extended factor to its originating factor (\cref{theorem:surjection}).
This information may be relevant for downstream analysis and is intentionally not collapsed during construction.\\
While the resulting extension is unique, the origin of a given extended factor (in terms of equivalent representations or alternative but redundant extensions) may not be unique.
This is expected and does not affect comparability, as all such representations induce the same probability measure on the common \ac{ms}. 
Actually, the use of the KL-divergence is rather stable for probability measures in such situations.
Crucially, the value of the KL divergence is independent of the specific Laplace extension, since Laplace extensions only introduce uniform refinements that do not alter the underlying distribution.
Formally, consider extending both $P_1,P_2$ from $\Omega=\mathcal{X}_{\boldsymbol{R}_{\text{ex}}}$ to $P^{\text{x}}_1,P^{\text{x}}_2$ on the state space $\Omega^{\text{x}} = \mathcal{X}_{\boldsymbol{R}_{\text{ex}}}\times \mathcal{X}_{R_{\text{new}}}$ via an additional "irrelevant" Laplace variable $R_{\text{new}}\notin \boldsymbol{R}_{\text{ex}}$.
Then
\begin{align*}
D_{\text{KL}}(P^{\text{x}}_1\Vert P^{\text{x}}_{2})=&\sum_{(\boldsymbol{r}^{\text{ex}},r^{\text{new}})\in\Omega^{\text{x}}} P^{\text{x}}_1(\boldsymbol{r}^{\text{x}})\log(P^{\text{x}}_1(\boldsymbol{r}^{\text{x}})/P^{\text{x}}_2(\boldsymbol{r}^{\text{x}}))\\
=&\sum_{(\boldsymbol{r}^{\text{ex}},r^{\text{new}})\in\Omega^{\text{x}}}\frac{P_1(\boldsymbol{r}^{\text{ex}})}{\vert\mathcal{X}_{R_{\text{new}}} \vert}\log\left(\frac{P_1(\boldsymbol{r}^{\text{ex}}) (1/\vert\mathcal{X}_{R_{\text{new}}} \vert)}{P_2(\boldsymbol{r}^{\text{ex}})(1/\vert\mathcal{X}_{R_{\text{new}}} \vert)}\right)\\
   = &\sum_{r^{\text{new}}\in \mathcal{X}_{R_{\text{new}}}}\sum_{\boldsymbol{r}^{\text{ex}}\in \mathcal{X}_{\boldsymbol{R}_{\text{ex}}}} \frac{P_1(\boldsymbol{r}^{\text{ex}})}{\vert\mathcal{X}_{R_{\text{new}}} \vert}\log(P_1(\boldsymbol{r}^{\text{ex}})/P_2(\boldsymbol{r}_i))\\
   =& \sum_{r^{\text{new}}\in \mathcal{X}_{R_{\text{new}}}}\frac{1}{\vert\mathcal{X}_{R_{\text{new}}}\vert}
   \sum_{\boldsymbol{r}^{\text{ex}}\in \mathcal{X}_{\boldsymbol{R}_{\text{ex}}}}P_1(\boldsymbol{r}^{\text{ex}})\log(P_1(\boldsymbol{r}^{\text{ex}})/P_2(\boldsymbol{r}^{\text{ex}}))\\
    =&\sum_{\boldsymbol{r}^{\text{ex}}\in \mathcal{X}_{\boldsymbol{R}_{\text{ex}}}} P_1(\boldsymbol{r}^{\text{ex}})\log(P_1(\boldsymbol{r}^{\text{ex}})/P_2(\boldsymbol{r}^{\text{ex}}))=D_{\text{KL}}(P_1\Vert P_{2}), 
\end{align*}
since the Laplace extension cancels multiplicatively.
Hence, KL divergence is invariant under arbitrary (even redundant) Laplace extensions, establishing scale invariance with respect to the \ac{ms}. 
This highlights a key strength of our framework: while arbitrary extensions may lead to different representations, \cref{algo:msx} enforces structural alignment and selects the minimal canonical form.
In contrast, non-structured extensions may share the same \ac{ms} but lack structural comparability, which is precisely what \cref{algo:msx} guarantees.\\

\section{Connection to Shannon Entropy}\label{append:BExample}
We can calculate the KL-divergence for the direct comparison of one distribution $P_M$ to a purely Laplace factor graph (implying an underlying fully uniform distribution $P_U$) by a simplified calculation using the Shannon-Entropy $H(P):=-\sum_{\boldsymbol{r}\in\mathcal{X}_{\boldsymbol{R}}} P(\boldsymbol{R}=\boldsymbol{r})\log(P(\boldsymbol{R}=\boldsymbol{r}))$. With $n$ being the number of different elements $\boldsymbol{r}$ in the state space, we get:
      \begin{align*}
        D_{\text{KL}}(P_M\Vert P_{U}) &= \sum_{\boldsymbol{r}\in\mathcal{X}_{\boldsymbol{R}}} P_M(\boldsymbol{r})\log\left(\frac{P_M(\boldsymbol{r})}{P_U(\boldsymbol{r})}\right)= \sum_{\boldsymbol{r}\in\mathcal{X}_{\boldsymbol{R}}} P_M(\boldsymbol{r})\log\left(\frac{P_M(\boldsymbol{r})}{1/n}\right)\\
        &=\sum_{\boldsymbol{r}\in\mathcal{X}_{\boldsymbol{R}}} P_M(\boldsymbol{r})\log(P_M(\boldsymbol{r}))+\sum_{\boldsymbol{r}\in\mathcal{X}_{\boldsymbol{R}}}P_M(\boldsymbol{r})\log(n)\\
        &=- H(P_M) +\log(n) \cdot 1\\
        &=-\sum_{\boldsymbol{r}\in\mathcal{X}_{\boldsymbol{R}}} \frac{1}{n}\log(1/n) - H(P_M)= H(P_U)-H(P_M)=\log(n)-H(P_M)
    \end{align*}
Due to $H(P_U)=\log(n)$, the calculation of $D_{\text{KL}}$ reduces directly to $H(P_M)$.
\section{Detailed Proofs}
\setcounter{theorem}{0}
\begin{theorem}
Let $M^{\text{x}}$ be an \ac{lx} of $M$.
Then, the partition function of $P_{M^{\text{x}}}$ is given by
\begin{align*}
        Z^{\text{x}}= c \cdot \prod_{X\in \boldsymbol{R}^{\text{new}}}^{} \vert \text{range}(X)\vert \cdot Z= c\cdot \vert \mathcal{X}_{\boldsymbol{R}^{\text{new}}}\vert\cdot Z
\end{align*}
    with $c:=\prod_{\phi^{\text{x}}_i\in\boldsymbol{\Phi}^{\text{x}}} c_i\in\mathbb{R}_{> 0}$ and $c_i$ given by $c_i\cdot \phi_i(\boldsymbol{r}^{\text{orig}}_i)=  \phi_i^{\text{x}}(\boldsymbol{r}_i^{\text{x}})$ for $\phi_i^{\text{x}}\in\boldsymbol{\Phi}^{\text{orig},\text{x}}$ and $\phi_{i}(\boldsymbol{r}^{\text{x}}_{i})=c_i$ for $\phi_i^{\text{x}}\in\boldsymbol{\Phi}^{\text{new}}$.
\end{theorem}
\newpage
\begin{proof}
    Due to normalisation of probability measures and with a unique pairwise assignment of indices from $\phi^{\text{x}}_i\in\boldsymbol{\Phi}^{\text{orig},\text{x}}$ to $\phi_i\in\boldsymbol{\Phi}^{\text{orig}} $, we get 
	\begin{align*}
		\hspace{-0.5cm}1 &=\sum_{\boldsymbol{r}^{\text{x}}\in\mathcal{X}_{\boldsymbol{R}^{\text{x}}}}P_{M^{\text{x}}}(\boldsymbol{r}^{\text{x}})\\
		&\overset{\text{def.}}{=}\sum_{\boldsymbol{r}^{\text{x}}\in\mathcal{X}_{\boldsymbol{R}^{\text{x}}}}\frac{1}{Z^{\text{x}}} \prod_{\phi_i^{\text{x}}\in \boldsymbol{\Phi}^{\text{x}}}^{} \phi_i^{\text{x}}(\boldsymbol{r}_i^{\text{x}})\\
        &= \frac{1}{Z^{\text{x}}}\sum_{\boldsymbol{r}^{\text{x}}\in\mathcal{X}_{\boldsymbol{R}^{\text{x}}}} \prod_{\phi_i^{\text{x}}\in \boldsymbol{\Phi}^{\text{orig},\text{x}}}^{} \phi_i^{\text{x}}(\boldsymbol{r}_i^{\text{x}})\prod_{\phi_j^{\text{x}}\in \boldsymbol{\Phi}^{\text{new}}}^{} \phi_j^{\text{x}}(\boldsymbol{r}_j^{\text{x}})\\
        &\overset{\text{Lap.}}{=}\frac{1}{Z^{\text{x}}}\sum_{\boldsymbol{r}^{\text{x}}\in\mathcal{X}_{\boldsymbol{R}^\text{x}}}\prod_{\phi_i^{\text{x}}\in \boldsymbol{\Phi}^{\text{orig},\text{x}}}^{} c_i\cdot \phi_i(\boldsymbol{r}^{\text{orig}}_i)\prod_{\phi_j^{\text{x}}\in \boldsymbol{\Phi}^{\text{new}}}^{} c_j\\
        &\overset{\text{ind.}}{=}\frac{1}{Z^{\text{x}}}\prod_{\phi_l^{\text{x}}\in \boldsymbol{\Phi}^{\text{orig},\text{x}}}^{} c_l\prod_{\phi_j^{\text{x}}\in \boldsymbol{\Phi}^{\text{new}}}^{}  c_j \sum_{\boldsymbol{r}^{\text{x}}\in\mathcal{X}_{\boldsymbol{R}^{\text{x}}}}\prod_{\phi_i^{\text{x}}\in \boldsymbol{\Phi}^{\text{orig},\text{x}}}^{}\phi_i(\boldsymbol{r}^{\text{orig}}_i)\\
        &=\frac{1}{Z^{\text{x}}}\prod_{\phi_j^{\text{x}}\in \boldsymbol{\Phi}^{\text{x}}}^{} c_j
        \sum_{\boldsymbol{r}^{\text{orig}}\in\mathcal{X}_{\boldsymbol{R}^{\text{orig}}}}\sum_{\boldsymbol{r}^{\text{new}}\in\mathcal{X}_{\boldsymbol{R}^{\text{new}}} }\prod_{\phi_i^{\text{x}}\in \boldsymbol{\Phi}^{\text{orig},\text{x}}}^{}\phi_i(\boldsymbol{r}^{\text{orig}}_i)\\
         &\overset{\text{ind.}}{=}\frac{1}{Z^{\text{x}}}\prod_{\phi_j^{\text{x}}\in \boldsymbol{\Phi}^{\text{x}}}^{} c_j
        \sum_{\boldsymbol{r}^{\text{orig}}\in\mathcal{X}_{\boldsymbol{R}^{\text{orig}}}}\prod_{\phi_i^{\text{x}}\in \boldsymbol{\Phi}^{\text{orig},\text{x}}}^{}\phi_i(\boldsymbol{r}^{\text{orig}}_i)\sum_{\boldsymbol{r}^{\text{new}}\in\mathcal{X}_{\boldsymbol{R}^{\text{new}}}} 1\\
        &=\frac{1}{Z^{\text{x}}}\cdot c
        \sum_{\boldsymbol{r}^{\text{orig}}\in\mathcal{X}_{\boldsymbol{R}^{\text{orig}}}}\prod_{\phi_i^{\text{x}}\in \boldsymbol{\Phi}^{\text{orig},\text{x}}}^{}\phi_i(\boldsymbol{r}^{\text{orig}}_i) \cdot \vert \mathcal{X}_{\boldsymbol{R}^{\text{new}}}\vert \\
           &=\frac{1}{Z^{\text{x}}}\cdot c \cdot \vert \mathcal{X}_{\boldsymbol{R}^{\text{new}}}\vert 
        \sum_{\boldsymbol{r}^{\text{orig}}\in\mathcal{X}_{\boldsymbol{R}^{\text{orig}}}}\prod_{\phi_i\in \boldsymbol{\Phi}^{\text{orig}}}^{}\phi_i(\boldsymbol{r}^{\text{orig}}_i)\\
        &\overset{\text{def.}}{=}\frac{1}{Z^{\text{x}}}\cdot c
         \cdot \vert \mathcal{X}_{\boldsymbol{R}^{\text{new}}}\vert \cdot Z\\
        \end{align*}\vspace{-0.85cm}
         \begin{align*}
 \hspace{-1.3cm}\overset{\text{Equ. Transf.}}{\Leftrightarrow} Z^{\text{x}}&= c  \cdot \vert \mathcal{X}_{\boldsymbol{R}^{\text{new}}}\vert \cdot Z\\
        &= c\cdot \vert \times_{X\in \boldsymbol{R}^{\text{new}}} \text{range}(X) \vert \cdot Z\\
        &=c\cdot \prod_{X\in \boldsymbol{R}^{\text{new}}}^{}\vert \text{range}(X)\vert \cdot Z. \qedhere
        \end{align*}
\end{proof}
    

\begin{corollary}
    If $M^{\text{x}}$ is an \ac{lx} of $M$ and $\boldsymbol{R}^{\text{new}}=\emptyset$, then $Z^{\text{x}} = c \cdot Z$, where $c \in \mathbb{R}_{> 0}$ is a scalar.
\end{corollary}
\begin{proof}
    The result follows by explicitly expanding the partition function of the extended model.
    For factors corresponding to original factors of $M$, each potential is scaled by a constant $c_i$, contributing a multiplicative factor $c=\prod_{\phi_i\in\boldsymbol{\Phi}^{\text{orig}}} c_i$.
    All terms are rescaling terms.
    Given that $\boldsymbol{R}^{\text{new}}=\emptyset$, per definition it holds that $\lvert \mathcal{X}_{\boldsymbol{R}^{\text{new}}}\rvert = 1$ (as the cardinality of the Cartesian product over empty sets is equal to one).
    Thus, applying \cref{theorem:generalpartitionfunction} yields
    \begin{align*}
        Z^x &= c \cdot \vert \mathcal{X}_{\boldsymbol{R}^{\text{new}}}\vert \cdot Z
        = c \cdot Z. \qedhere
    \end{align*}
\end{proof}

\newpage
\setcounter{theorem}{2}
\begin{corollary}
    Let $M$ be a Laplace \ac{fg} and let $\kappa_X:=\lvert\{\phi_i: X\in \boldsymbol{R}^{\text{orig}}_{(i)}\}\rvert$ denote the number of factors in which \ac{rv} $X$ appears.
Then $M$ can be represented by the following scaled factors
\begin{align*}
\phi_i(\boldsymbol{r}_i)=\prod_{X\in \boldsymbol{R}^{\text{orig}}_{(i)}}\vert \text{range}(X)\vert ^{- 1/\kappa_X } 
\end{align*}
for all $\boldsymbol{r}^{\text{orig}}_i\in\mathcal{X}_{\boldsymbol{R}^{\text{orig}}_{(i)}}$ and with $Z=1$. 
\end{corollary}
\begin{proof}
Since $M$ is a Laplace \ac{fg}, all factors are Laplace, which means every factor $\phi_i$ has the same value for all its potentials and there exists a constant $c_i\in \mathbb{R}_{>0}$:
        \begin{align*}
            \phi_{i}(\boldsymbol{r}^{\text{orig}}_{i})=c_i\text{ for all }\boldsymbol{r}^{\text{orig}}_i\in \mathcal{X}_{\boldsymbol{R}^{\text{orig}}_{(i)}}.
        \end{align*}
We choose $c_i:=\prod_{X\in \boldsymbol{R}^{\text{orig}}_{(i)}}\vert X\vert ^{- 1/\kappa_X } $ for the $i$-th factor and show that it normalises $Z$ to $1$.
\begin{align*}
        1&= \sum_{\boldsymbol{r}^{\text{orig}}\in\mathcal{X}_{\boldsymbol{R}^{\text{orig}}}}  P_M(\boldsymbol{r}^{\text{orig}})\\
        &=\frac{1}{Z}\sum_{\boldsymbol{r}^{\text{orig}}\in\mathcal{X}_{\boldsymbol{R}^{\text{orig}}}}\prod_{\phi_i\in \boldsymbol{\Phi}^{\text{orig}}}^{}\phi_i(\boldsymbol{r}^{\text{orig}}_i)\\
        &= \frac{1}{Z}\sum_{\boldsymbol{r}^{\text{orig}}\in\mathcal{X}_{\boldsymbol{R}^{\text{orig}}}}\prod_{\phi_i\in \boldsymbol{\Phi}^{\text{orig}}}^{}\prod_{X_l\in \boldsymbol{R}^{\text{orig}}_{(i)}}\vert \text{range}(X_l)\vert ^{- 1/\kappa_{X_l} }\\
        &\overset{\text{ind.}}{=}  \frac{1}{Z}\prod_{\phi_i\in \boldsymbol{\Phi}^{\text{orig}}}^{}\prod_{X_l\in \boldsymbol{R}^{\text{orig}}_{(i)}}\vert \text{range}(X_l)\vert ^{- 1/\kappa_{X_l} }\sum_{\boldsymbol{r}^{\text{orig}}\in\mathcal{X}_{\boldsymbol{R}^{\text{orig}}}}1\\
        &= \frac{1}{Z}\prod_{\phi_i\in \boldsymbol{\Phi}^{\text{orig}}}^{}\prod_{X_l\in \boldsymbol{R}^{\text{orig}}_{(i)}}\vert \text{range}(X_l)\vert ^{- 1/\kappa_{X_l} } \cdot \vert \mathcal{X}_{\boldsymbol{R}^{\text{orig}}}\vert\\
    &\overset{\text{ind.}}{=}\frac{\vert \mathcal{X}_{\boldsymbol{R}^{\text{orig}}}\vert}{Z}\prod_{\phi_i\in \boldsymbol{\Phi}^{\text{orig}}}^{}\prod_{X_l\in \boldsymbol{R}^{\text{orig}}_{(i)}}\vert \text{range}(X_l)\vert ^{- 1/\kappa_{X_l} } \\
        &= \frac{\vert \mathcal{X}_{\boldsymbol{R}^{\text{orig}}}\vert}{Z}\cdot \underbrace{\vert \text{range}(X_1) \vert^{- 1/\kappa_{X_1} }\cdot\ldots\cdot\vert \text{range}(X_1) \vert ^{- 1/\kappa_{X_1} }}_{\kappa_{X_1}\text{times}} \times\cdots\\
        &\qquad\cdots\times\underbrace{\vert \text{range}(X_{\rvert \boldsymbol{R}^{\text{orig}}\rvert}) \vert^{- 1/\kappa_{X_{\rvert \boldsymbol{R}^{\text{orig}}\rvert}} }\cdot\ldots
        \cdot\vert \text{range}(X_{\rvert \boldsymbol{R}^{\text{orig}}\rvert})\vert^{- 1/\kappa_{X_{\rvert \boldsymbol{R}^{\text{orig}}\rvert}} } }_{\kappa_{X_{\rvert \boldsymbol{R}^{\text{orig}}\rvert}}\text{times}}\\
        &= \frac{\vert \mathcal{X}_{\boldsymbol{R}^{\text{orig}}}\vert}{Z} \cdot \vert \text{range}(X_1) \vert^{- 1}\cdot\ldots\cdot \vert \text{range}(X_{\rvert \boldsymbol{R}^{\text{orig}}\rvert}) \vert^{- 1}\\
        &= \frac{\vert \mathcal{X}_{\boldsymbol{R}^{\text{orig}}}\vert}{Z}\prod_{i=1}^{\rvert \boldsymbol{R}^{\text{orig}}\rvert}\vert \text{range}(X_i) \vert^{- 1}\\
        &=\frac{\vert \mathcal{X}_{\boldsymbol{R}^{\text{orig}}}\vert}{Z}\cdot\vert \mathcal{X}_{\boldsymbol{R}^{\text{orig}}}\vert^{-1}\\
        &=\frac{1}{Z} 
\end{align*}
\end{proof}
\newpage
\begin{theorem}
    Let $M^{\text{x}}$ be an \ac{lx} of an \ac{fg} $M$. Then,
    \begin{align*}
              P_M= P_{M^{\text{x}}}\;\circ\;\pi^{-1}.
    \end{align*}
\end{theorem}

\begin{proof}
It is sufficient to use equality transformations using the given properties and \cref{theorem:generalpartitionfunction}.
\begin{align*}
P_{M^{\text{x}}}(\pi^{-1}\left(\boldsymbol{r}^{\text{orig}}\right))&=P_{M^{\text{x}}}(\cup_{\boldsymbol{r}^{\text{new}}\in\mathcal{X}_{\boldsymbol{R}^\text{new}}}\{(\boldsymbol{r}^{\text{orig}},\boldsymbol{r}^{\text{new}})\}) \\
        &\overset{\text{dis.}}{=} \sum_{\boldsymbol{r}^{\text{new}}\in\mathcal{X}_{\boldsymbol{R}^{\text{new}}}} P_{M^{\text{x}}}\left(\{(\boldsymbol{r}^{\text{orig}},\boldsymbol{r}^{\text{new}})\}\right)\\
	    &\overset{\text{def.}}{=}\frac{1}{Z^{\text{x}}} \sum_{\boldsymbol{r}^{\text{new}}\in\mathcal{X}_{\boldsymbol{R}^{\text{new}}}}\prod_{\phi_{i}^{\text{x}}\in \boldsymbol{\Phi}^{\text{x}}}^{}\phi_{i}^{\text{x}}\left((\boldsymbol{r}^{\text{orig}}_i,\boldsymbol{r}^{\text{new}}_i)\right)\\
    	&=\frac{1}{Z^{\text{x}}}\sum_{\boldsymbol{r}^{\text{new}}\in\mathcal{X}_{\boldsymbol{R}^{\text{new}}}} \prod_{\phi_{i}^{\text{orig},\text{x}}\in \boldsymbol{\Phi}^{\text{orig},\text{x}}}^{}\phi_{i}^{\text{x}}\left((\boldsymbol{r}^{\text{orig}}_i,\boldsymbol{r}^{\text{new}}_i)\right)
        \prod_{\phi_{j}^{\text{x}}\in \boldsymbol{\Phi}^{\text{new}}}^{}\phi_{j}^{\text{x}}\left((\boldsymbol{r}^{\text{orig}}_j,\boldsymbol{r}^{\text{new}}_j)\right)\\
	&\overset{\text{Lapl.}}{=} \frac{1}{Z^{\text{x}}}\sum_{\boldsymbol{r}^{\text{new}}\in\mathcal{X}_{\boldsymbol{R}^{\text{new}}}}\prod_{\phi_{i}^{\text{orig},\text{x}}\in \boldsymbol{\Phi}^{\text{orig},\text{x}}}^{}\phi_{i}(\boldsymbol{r}^{\text{orig}}_i)\cdot c_i\cdot \prod_{\phi_{j}^{\text{x}}\in \boldsymbol{\Phi}^{\text{new}}}^{} c_j\\
    &\overset{\cref{theorem:generalpartitionfunction}}{=} \frac{\sum_{\boldsymbol{r}^{\text{new}}\in\mathcal{X}_{\boldsymbol{R}^{\text{new}}}}1}{Z\cdot c   
    \cdot \vert \mathcal{X}_{\boldsymbol{R}^{\text{new}}}\vert}\;\cdot \prod_{\phi_{i}^{\text{orig},\text{x}}\in \boldsymbol{\Phi}^{\text{orig},\text{x}}}^{}\phi_{i}(\boldsymbol{r}^{\text{orig}}_i)\cdot \prod_{\phi_{j}^{\text{x}}\in \boldsymbol{\Phi}^{\text{x}}}^{} c_j\\
    &\overset{\text{Def. of } c}{=}\frac{ \vert \mathcal{X}_{\boldsymbol{R}^{\text{new}}}\vert}{Z\cdot \vert \mathcal{X}_{\boldsymbol{R}^{\text{new}}}\vert}\;\cdot \prod_{\phi_{i}^{\text{orig},\text{x}}\in \boldsymbol{\Phi}^{\text{orig},\text{x}}}\phi_{i}(\boldsymbol{r}^{\text{orig}}_i)\\
    &\overset{\text{same}}{\underset{\text{index}}{=}}\frac{1}{Z}\;\cdot \prod_{\phi_{i}\in \boldsymbol{\Phi}^{\text{orig}}}^{}\phi_{i}(\boldsymbol{r}^{\text{orig}}_i)\\
    &\overset{\text{def.}}{=}P_M\left(\boldsymbol{r}^{\text{orig}}\right)\tag*{\qedhere}
\end{align*}
\end{proof}
\begin{lemma}
    Let $M^{\text{x}}$ be a factor graph extension of $M$.
    If every factor $\phi^{\text{x}}_i\in \boldsymbol{\Phi}^{\text{new}}$ is an independent extension of $M$ and $\boldsymbol{\Phi}^{\text{orig},\text{x}}=\boldsymbol{\Phi}^{\text{orig}}$, then the partition function of the extended \ac{fg} $M^{\text{x}}$ is given by
	\begin{align*}
		Z^{\text{x}} =Z  \sum_{\boldsymbol{r}^{\text{new}}\in\mathcal{X}_{\boldsymbol{R}^{\text{new}}}}\prod_{\phi_i^{\text{x}}\in\boldsymbol{\Phi}^{\text{new}}} \phi_i^{\text{x}}(\boldsymbol{r}_i^{\text{new}})
	\end{align*}
\end{lemma}
\begin{proof}
\begin{align*}
\hspace{-0.3cm}1&= \sum_{\boldsymbol{r}^{\text{x}}\in \mathcal{X}_{\boldsymbol{R}^{\text{x}}}}P_{M^{\text{x}}}(\boldsymbol{r}^{\text{x}})\\
      &= \frac{1}{Z^{\text{x}}}\sum_{\boldsymbol{r}^{\text{x}}\in\mathcal{X}_{\boldsymbol{R}^{\text{x}}}}^{} \prod_{\phi_i^{\text{x}}\in\boldsymbol{\Phi}^{\text{x}}}\phi_i^{\text{x}}(\boldsymbol{r}^{\text{x}}_i)\\
      &=\frac{1}{Z^{\text{x}}}\sum_{\boldsymbol{r}^{\text{x}}\in \mathcal{X}_{\boldsymbol{R}^{\text{x}}}}\prod_{\phi_i^{\text{x}}\in \boldsymbol{\Phi}^{\text{orig},\text{x}}}\phi_i^{\text{x}}(\boldsymbol{r}^{\text{x}}_i)\prod_{\phi_i^{\text{x}}\in \boldsymbol{\Phi}^{\text{new}}}\phi_i^{\text{x}}(\boldsymbol{r}^{\text{x}}_i)\\
      &\overset{\text{ind.}}{=}\frac{1}{Z^{\text{x}}}\sum_{\boldsymbol{r}^{\text{orig}}\in \mathcal{X}_{\boldsymbol{R}^{\text{orig}}}}
      \sum_{\boldsymbol{r}^{\text{new}}\in \mathcal{X}_{\boldsymbol{R}^{\text{new}}}}
      \prod_{\phi_i^{\text{x}}\in \boldsymbol{\Phi}^{\text{orig},\text{x}}}\phi_i^{\text{x}}(\boldsymbol{r}^{\text{orig}}_i)\prod_{\phi_i^{\text{x}}\in \boldsymbol{\Phi}^{\text{new}}}\phi_i^{\text{x}}(\boldsymbol{r}^{\text{new}}_i)\\
      &\overset{\text{iden.}}{=}\frac{1}{Z^{\text{x}}}\sum_{\boldsymbol{r}^{\text{orig}}\in \mathcal{X}_{\boldsymbol{R}^{\text{orig}}}}\sum_{\boldsymbol{r}^{\text{new}}\in \mathcal{X}_{\boldsymbol{R}^{\text{new}}}}\prod_{\phi_i\in \boldsymbol{\Phi}^{\text{orig}}}\phi_i(\boldsymbol{r}^{\text{orig}}_i)\prod_{\phi_i^{\text{x}}\in \boldsymbol{\Phi}^{\text{new}}}\phi_i^{\text{x}}(\boldsymbol{r}^{\text{new}}_i)\\
     &\overset{\text{ind.}}{=}\frac{1}{Z^{\text{x}}}\sum_{\boldsymbol{r}^{\text{orig}}\in \mathcal{X}_{\boldsymbol{R}^{\text{orig}}}} \prod_{\phi_i\in \boldsymbol{\Phi}^{\text{orig}}}\phi_i(\boldsymbol{r}^{\text{orig}}_i)
     \sum_{\boldsymbol{r}^{\text{new}}\in \mathcal{X}_{\boldsymbol{R}^{\text{new}}}}\prod_{\phi_i^{\text{x}}\in \boldsymbol{\Phi}^{\text{new}}}\phi_i^{\text{x}}(\boldsymbol{r}^{\text{new}}_i)\\
    &\overset{\text{def.}}{=}\frac{Z}{Z^{\text{x}}}
     \sum_{\boldsymbol{r}^{\text{new}}\in \mathcal{X}_{\boldsymbol{R}^{\text{new}}}}\prod_{\phi_i^{\text{x}}\in \boldsymbol{\Phi}^{\text{new}}}\phi_i^{\text{x}}(\boldsymbol{r}^{\text{new}}_i)    
\end{align*}
\end{proof}

\begin{theorem}
The \acp{fg} $M^{\text{x}}$ and $M'^{\text{x}}$ returned by \textit{Alg. 1} for input two \acp{fg} $M$ and $M'$ are structurally identical and encode probability distributions on the same \ac{ms}:
\begin{align*}
    \left(\mathcal{X}_{\boldsymbol{R}^{\text{orig}}\;\cup\;\boldsymbol{R'}^{\text{orig}}}, \mathcal{P}\left(\mathcal{X}_{\boldsymbol{R}^{\text{orig}}\;\cup\;\boldsymbol{R'}^{\text{orig}}}\right)\right)
\end{align*}
\end{theorem}

\begin{proof}
For every individual factor $\phi_i\in\boldsymbol{\Phi}^{\text{orig}}$ or $\phi'_j\in\boldsymbol{\Phi'}^{\text{orig}}$, there is an Laplace extension step within Alg. 1, which means that the outcomes $M^{\text{x}}$ and $M'^{\text{x}}$ are indeed extended \acp{fg}.
The steps include trivial extensions when, e.g., some factor is subsumed in the considered factor or for structurally identical factors.

In addition, $\boldsymbol{R}^{\text{orig}}_{(i)}\subset \boldsymbol{R}^{\text{x}}_{(i)}$ holds, resulting in $\boldsymbol{R}^{\text{orig}}\subset \boldsymbol{R}^{\text{x}}$.

Let $X\in \boldsymbol{R'}^{\text{orig}}\setminus \boldsymbol{R}^{\text{orig}}$, then there exists a factor $\phi'_j\in \boldsymbol{\Phi'}$ such that $X\in \boldsymbol{R'}^{\text{orig}}_{\kern-0.4em  (j)}\setminus \boldsymbol{R}^{\text{orig}}_{(i)} $ for all $i=1,\ldots,\boldsymbol{\Phi}^{\text{orig}}$.
This can only happen, when either $\phi'_j$ is an independent factor to all $\phi_i$ or it is overlapping with at least one of the $\phi_i$ or it exists $\phi_i$, which is subsumed in $\phi'_j$.
For all cases, an extension is added to $M$ containing $\boldsymbol{R}^{\text{orig}}_{ (j)}$, leading to $X\in \boldsymbol{R}^{\text{new}}\supset \boldsymbol{R'}^{\text{orig}}\setminus  \boldsymbol{R}^{\text{orig}}$.

By construction of \cref{algo:msx}, it is not possible that $\boldsymbol{R}^{\text{new}}\setminus \boldsymbol{R'}^{\text{orig}} \neq \{\}$ and therefore $\boldsymbol{R}^{\text{x}}=\boldsymbol{R}^{\text{orig}}\cup \boldsymbol{R'}^{\text{orig}}$ holds .

Every individual extension step applied to $M$ is a \ac{lx} by construction, which makes \cref{theorem:generalpartitionfunction} applicable.
Therefore, any single assignment $\boldsymbol{r}^{\text{x}}\in \boldsymbol{R}^{\text{x}}$ is uniquely defined via
\begin{align*}
    P_{M^{\text{x}}}(\boldsymbol{r}^{\text{x}}) & = \frac{1}{Z^{\text{x}}} \prod_{\phi_i^{\text{x}}\in \boldsymbol{\Phi}^{\text{x}}} \phi_i^{\text{x}}(\boldsymbol{r}_i^{\text{x}}) 
\end{align*}
with $Z^{\text{x}}= Z\cdot c \prod_{X\in \boldsymbol{R}^{\text{new}}}^{} \vert \range(X)\vert= Z\cdot c\cdot \vert \mathcal{X}_{\boldsymbol{R}^{\text{new}}}\vert$, $c:=\prod_{\phi^{\text{x}}_i\in\boldsymbol{\Phi}^{\text{x}}} c_i\in\mathbb{R}_{> 0}$, $c_i$ given by $c_i\cdot \phi_i(\boldsymbol{r}^{\text{orig}}_i)=  \phi_i^{\text{x}}(\boldsymbol{r}_i^{\text{x}})$ for $\phi_i^{\text{x}}\in\boldsymbol{\Phi}^{\text{orig},\text{x}}$, and $\phi_{i}(\boldsymbol{r}^{\text{x}}_{i})=c_i$ for $\phi_i^{\text{x}}\in\boldsymbol{\Phi}^{\text{new}}$.
Which means $P_{M^{\text{x}}}$ defines a probability distribution on $(\mathcal{X}_{\boldsymbol{R}^{\text{orig}}\;\cup\;\boldsymbol{R'}^{\text{orig}}}, \mathcal{P}(\mathcal{X}_{\boldsymbol{R}^{\text{orig}}\;\cup\;\boldsymbol{R'}^{\text{orig}}}))$.

It remains to show structural identity by constructing a bijection $\beta:\boldsymbol{\Phi}^{\text{x}}\rightarrow\boldsymbol{\Phi'}^{\text{x}}$ as required in Def. 5.
Recall, we assume that neither of the original \acp{fg} $M$ nor $M'$ do not contain structurally identical factors within itself.

For each original factor $\phi_i\in\boldsymbol\Phi^{\text{orig}}$,
the mapping is defined as follows:
\begin{itemize}
    \item If $\phi_i$ and some $\phi'_j\in\boldsymbol\Phi'^{\text{orig}}$
    are structurally identical or one subsumes the other, then
    $\beta(\phi_i^{\text{x}})=\phi_j'^{\text{x}}$.
    \item If $\phi_i$ is independent of all factors in $M'$,
    then a Laplace factor $\phi_j'^{\text{x}}\in\boldsymbol\Phi'^{\text{new}}$
    with identical scope is introduced and
    $\beta(\phi_i^{\text{x}})=\phi_j'^{\text{x}}$.
    The symmetric case is handled analogously.
    \item For every overlapping pair $\phi_i$ and $\phi'_j$,
    Alg. 1 introduces Laplace extensions of both factors
    to the united scope, and
    $\beta(\phi_i^{\text{x}})=\phi_j'^{\text{x}}$.
\end{itemize}

By construction, all matched factor pairs share the same scope
$\boldsymbol R^{\text{x}}_{(i)}=\boldsymbol R'^{\text{x}}_{(j)}$.
Hence $\beta$ is well-defined and bijective, and
$M^{\text{x}}$ and $M'^{\text{x}}$ are structurally identical.
\end{proof}













\end{document}